\documentclass{article}
\usepackage{PRIMEarxiv}

\usepackage{todonotes}
\usepackage{multirow}
\usepackage{algorithm}
\usepackage{algpseudocode}
\usepackage{amsthm}
\newtheorem{theorem}{Theorem}

\newtheorem{lemma}{Lemma}
\newtheorem{remark}{Remark}

\makeatletter
\renewenvironment{proof}[1][\proofname]{\par
  \vspace{-\topsep}
  \pushQED{\qed}%
  \normalfont
  \topsep0pt \partopsep0pt 
  \trivlist
  \item[\hskip\labelsep
        \itshape
    #1\@addpunct{.}]\ignorespaces
}{%
  \popQED\endtrivlist\@endpefalse
  \addvspace{6pt} 
}
\makeatother

\usepackage[utf8]{inputenc} 
\usepackage[T1]{fontenc}    
\usepackage{hyperref}       
\usepackage{url}            
\usepackage{booktabs}       
\usepackage{amsfonts}       
\usepackage{nicefrac}       
\usepackage{microtype}      
\usepackage{xcolor}         
\usepackage{caption}
\usepackage{subcaption}
\usepackage[round]{natbib}

\usepackage{amsmath, bm, amssymb}
\usepackage{empheq}
\PassOptionsToPackage{pdftex}{graphicx}

\title{Global optimization of graph acquisition functions for neural architecture search
}

\author{
    Yilin Xie,~~Shiqiang Zhang,~~Jixiang Qing,~~Ruth Misener,~~Calvin Tsay$^\star$ \\
    Department of Computing, Imperial College London \\
    London, United Kingdom
}

\begin{document}
\maketitle

\def\thefootnote{$\star$}\footnotetext{Corresponding author: \texttt{c.tsay@imperial.ac.uk}}

\begin{abstract}
    Graph Bayesian optimization (BO) has shown potential as a powerful and data-efficient tool for neural architecture search (NAS). Most existing graph BO works focus on developing graph surrogates models, i.e., metrics of networks and/or different kernels to quantify the similarity between networks. However, the acquisition optimization, as a discrete optimization task over graph structures, is not well studied due to the complexity of formulating the graph search space and acquisition functions. This paper presents explicit optimization formulations for graph input space including properties such as reachability and shortest paths, which are used later to formulate graph kernels and the acquisition function. We theoretically prove that the proposed encoding is an equivalent representation of the graph space and provide restrictions for the NAS domain with either node or edge labels. Numerical results over several NAS benchmarks show that our method efficiently finds the optimal architecture for most cases, highlighting its efficacy.
\end{abstract}

\section{Introduction}\label{sec:introduction}

Despite numerous breakthroughs in deep learning, the design of neural architectures largely relies on prior experience and heuristic search. 
Moreover, the neural architecture design underlies the learned representation of the data, and the ultimate downstream performance in predictive tasks. 
The field of neural architecture search (NAS) seeks to automate this key step, by letting algorithms automatically design the architecture of a neural network model~\citep{ren2021comprehensive}. 
In general, NAS algorithms share several steps~\citep{salmani2025systematic}: (i) encoding the search space, e.g., as a general or modular domain, (ii) prescribing a search strategy over the above space, and (iii) assessing the (approximate) performance at selected points. Early works in NAS sought to encode a general search space from scratch, e.g., as a string~\citep{zoph2017neural}. 
Later works constrain the search space toward problem tractability, such as by explicitly encoding a layer- or module-based structure~\citep{liu2018progressive, wu2019fbnet}. Search strategies are often based on random search, gradient-based optimization~\citep{liu2019darts,wu2019fbnet}, Bayesian optimization~\citep{white2021bananas, ru2021interpretable}, evolutionary algorithms~\citep{real2019evolution, qiu2023selfevol}, or reinforcement learning~\citep{zoph2017neural, jaafra2019rlreview, cheng2022dpnas}. 
Finally, performance assessments are the most expensive step of NAS, often involving full or partial training of the proposed model(s).

Graph Bayesian optimization (BO) exhibits state-of-the-art performance in  NAS \citep{elsken2019NASsurvey, white2023NASthousand}, given the ability of algorithms to efficiently explore the graph search space and identify promising architectures within limited budgets \citep{ru2021interpretable}. Graph BO addresses the above NAS steps using (i) a graph surrogate that is trained over available data and then serves as a predictor, and (ii) an acquisition function, encoding trade-offs between exploitation and exploration, that is optimized to propose the next candidate. From modeling perspectives, Gaussian processes (GPs) \citep{schulz2018GP} are commonly used since they offer accurate prediction along with uncertainty quantification. To apply GPs in a graph domain, graph kernels \citep{vishwanathan2010graph,borgwardt2020graph, kriege2020survey,nikolentzos2021graph} are introduced to measure the similarity between graphs. Although advances in graph kernels facilitate the generalization from non-structural spaces to graph space, optimizing acquisition functions over (combinatorial) graph spaces remains a challenge. Most works use sample-based or evolutionary algorithms, since they only require evaluations of the acquisition function and can thus be easily applied to various graph domains. However, these algorithms must incorporate problem-specific constraints into the sampling and mutation steps to remove invalid candidates, and there is no theoretical guarantee about the optimality of solutions obtained from these methods.

Recently, the idea of using mathematical programming techniques to formulate machine learning (ML) models, e.g., neural networks (NNs) \citep{fischetti2018deep,anderson2020strong,tsay2021partition,zhang2023optimizing}, trees \citep{misic2020optimization,mistry2021mixed,ammari2023linear}, and GPs \citep{schweidtmann2021deterministic,xie2024global}, has attracted a lot of attention, since it provides a way to explicitly solve decision-making problems involving ML models. Relevant applications include BO acquisition optimization \cite{thebelt2021entmoot,thebelt2022maximizing, wang2023optimizing}, NN verification \citep{huchette2023deep, hojny2024verifying}, molecular design \citep{zhang2024augmenting,mcdonald2024mixed}, among others. Based on the global optimization formulation for acquisition optimization proposed in \citep{xie2024global}, \citet{xie2025bogrape} propose BoGrape as a general graph BO framework, comprising the first work to treat graph acquisition functions from a discrete optimization viewpoint. By encoding graph spaces and shortest-path graph kernels \citep{borgwardt2005shortest} into mixed-integer programming (MIP), BoGrape can handle constraints over graph search spaces and globally optimize the acquisition function with theoretical guarantees. However, the requirement of strong connectivity makes BoGrape unsuitable for NAS, since neural architectures are weakly connected acyclic digraphs (DAGs). 

\begin{figure}
    \centering
    \includegraphics[width=\linewidth]{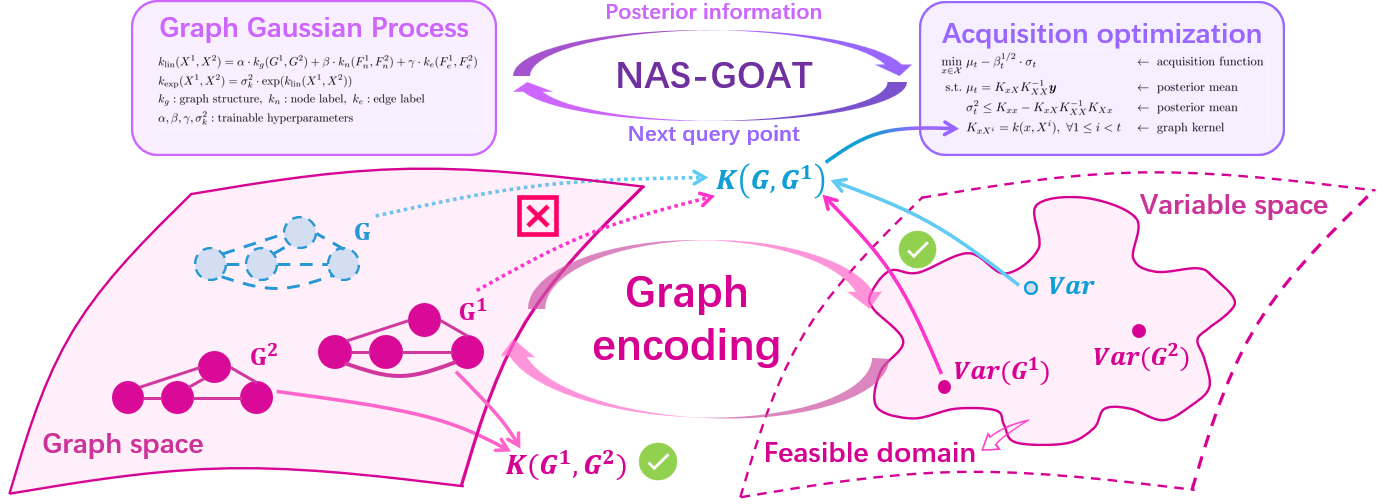}
    \caption{Illustration of NAS-GOAT. The main idea is to represent graphs in variable space and introduce constraints to build a bijection between all graphs and the feasible domain. The graph kernel value between an unknown graph (which is our optimization target) and a given graph is then formulated as expressions of variables, or constraints, enabling us to employ global optimization for acquisition function and propose the next neural architecture to evaluate. }
    \label{fig:illustration}
\end{figure}

This paper studies the global optimization of graph acquisition functions for graph BO-based NAS. To represent the graph space containing valid neural architectures, we generalize the graph encoding presented in \citet{xie2025bogrape} to omit assumptions about connectivity, and we show how the resulting general encoding can be restricted to the NAS search space. GPs with shortest-path kernels are used as graph surrogates, and lower confidence bound (LCB) \citep{srinivas2010gaussian} is chosen as the acquisition function. The proposed graph encoding contains graph properties including reachability and shortest paths and is therefore compatible with existing formulations for shortest-path graph kernels and acquisition functions \citep{xie2025bogrape}. The final acquisition optimization is formulated as a MIP, which can be solved by MIP solvers with global optimality guarantees. Figure~\ref{fig:illustration} illustrates the main idea of the proposed framework, we also list the major contributions of this work are as follows:
\begin{itemize}
    \item We present an equivalent representation for general labeled graphs in variable space. Each graph corresponds to a unique feasible solution containing its graph structure, as well as graph properties including reachability, shortest distances, and shortest paths.
    \item We provide a general kernel form measuring the similarity between two labeled graphs over graph structure, node label, and edge label levels, and we present a formulation that is compatible with our graph encoding.
    \item We incorporate NAS-specific constraints to the graph encoding, which define the valid search space for the NAS task using either node labels or edge labels. 
    \item We propose NAS-GOAT to \textbf{g}lobally \textbf{o}ptimize graph \textbf{a}cquisi\textbf{t}ion functions based on our proposed encoding. Numerical results demonstrating a full BO loop on NAS benchmarks show the efficiency and potential of NAS-GOAT.  
\end{itemize}

\textbf{Paper structure:} Section \ref{sec:background} provides preliminary knowledge. Section \ref{sec:methodology} introduces our methodology and theoretical results. Section  \ref{sec:experiments} reports the experimental results. Section \ref{sec:conclusion} concludes this work and discusses the future work. Appendices contain all proofs and further details.

\section{Background}\label{sec:background}
\subsection{Cell-based NAS}\label{subsec:NAS}
In many NAS search spaces, a network architecture is designed by varying some repeated small feedforward sub-structures called cells \citep{ying2019bench, dong2020bench}. Each cell is treated as a DAG, where the operation units are represented as node or edge labels, and information flows within the cell following graph topologies. Cells are then stacked multiple times and embedded into a macro neural network `skeleton' to give the final architecture. For instance, NAS-Bench-101 \citep{ying2019bench} and NAS-Bench-201 \citep{dong2020bench} define one stack as $3$ and $5$ replications of cells, resp., and each stack appears $3$ times in the overall network structure. Cell-based NAS can be naturally considered as an expensive black-box optimization problem, where one aims to search for the best graph, i.e., cell, that optimizes the performance of the resulting neural architecture over certain metrics, e.g., validation/test accuracy.

\subsection{Graph Bayesian optimization}\label{subsec:graph_BO}
Graph BO is a natural extension of BO \citep{frazier2018BO, garnett2023BObook} from vector space to graph space. At the $t$-th iteration, a graph Gaussian process (GP) equipped with a graph kernel is trained on available data $X=\{(G^i,F^i), y^i\}_{i=1}^{t-1}$. The posterior mean $\mu_t(\cdot)$ and variance $\sigma_t^2(\cdot)$ obtained from the graph GP are used to define acquisition functions such as lower confidence bound (LCB): $\alpha_{LCB}(x)=\mu_t(x)-\beta_t^{1/2}\cdot \sigma_t(x)$, where $\beta_t$ is a hyperparameter balancing between exploitation and exploration. 

From modeling perspectives, the core component of graph GPs is the graph kernels that measure the similarity between graphs. Classic graph kernels include random walk (RW) \citep{gartner2003graph}, subgraph matching (SM) \citep{kriege2012subgraph}, shortest-path (SP) \citep{borgwardt2005shortest}, Weisfeiler-Lehman (WL) \citep{shervashidze2011weisfeiler}, and Weisfeiler-Lehman optimal transport (WLOA) \citep{kriege2016valid} kernels; we refer the reader to \citet{vishwanathan2010graph,borgwardt2020graph, kriege2020survey,nikolentzos2021graph} for comprehensive details about graph kernels. In this work, we consider SP kernels used for graph BO in \citep{xie2025bogrape}. Mathematically, for two node labeled graphs $G^1$ and $G^2$, denote $V^1$ and $V^2$ as their node sets, resp., $l_v$ as the label of $v$, and $d_{u,v}$ as the shortest distance from node $u$ to node $v$. The SP kernel is defined as:
\begin{equation}\label{eq:SP_kernel}\tag{$k_g$}
    \begin{aligned}
        k_\mathit{SP}(G^1,G^2)=\frac{1}{n_1^2n_2^2}\sum\limits_{u_1,v_1\in V^1,u_2,v_2\in V^2}\mathbf 1(l_{u_1}=l_{u_2})\cdot \mathbf 1(d_{u_1,v_1}=d_{u_2,v_2})\cdot \mathbf 1(l_{v_1}=l_{v_2}),
    \end{aligned}    
\end{equation}
where $n_1^2n_2^2$ is a normalizing coefficient with $n_1$ and $n_2$ as the node number of $G^1$ and $G^2$, resp.

\subsection{Graph acquisition optimization}\label{subsec:acquisition_optimization}
The major challenge of graph BO is the acquisition optimization, which seeks to find the graph structure with optimal acquisition function value and is often required for convergence proofs. Encoding a graph search space and acquisition function as optimization constraints is non-trivial, and most existing works follow a sample-then-evaluate procedure to avoid directly optimizing over discrete space, e.g., \citep{kandasamy2018transport,ru2021interpretable,wan2021adversarial,wan2023bayesian}. From a discrete optimization viewpoint, \citet{xie2025bogrape} first formulate graph space and shortest-path graph kernels using MIP, and propose BoGrape as a graph BO framework that can globally optimize the lower confidence bound (LCB) acquisition:
\begin{equation}\label{eq:BoGrape}\tag{Acq-Opt}
    \begin{aligned}
        \min\limits_{x\in\mathcal X}&~\mu_t-\beta_t^{1/2}\cdot \sigma_t && \gets~\text{acquisition function}\\
        \text{s.t.}&~\mu_t=K_{xX}K_{XX}^{-1}\bm{y} && \gets~\text{posterior mean}\\
        &~\sigma^2_t\le K_{xx}-K_{xX}K_{XX}^{-1}K_{Xx} && \gets~\text{posterior mean}\\
        &~K_{xX^i}=k(x, X^i),~\forall 1\le i<t &&\gets~\text{graph kernel}
    \end{aligned}
\end{equation}
However, graphs considered in \citep{xie2025bogrape} are assumed to be strongly connected and have node labels, while graphs involved in NAS are weakly connected graphs, probably with edge labels.

\section{Methodology}\label{sec:methodology}
\subsection{Encoding a graph search space in optimization}\label{subsec:graph_encoding}
Firstly and most importantly, we must properly define the graph search space over which acquisition optimization is performed. In this section, we temporarily ignore node/edge features and focus on graph structures. To avoid graph isomorphism caused by node indexing, we assume that all nodes are labeled differently. Intuitively, encoding such a general graph space is easy, since each graph is uniquely determined by its adjacency matrix, and one only needs to define the $n\times n$ adjacency matrix containing binary variables $A_{u,v}$ that denote the existence of edge $u\to v$. However, this naive encoding has no extra graph information, e.g., connectivity, reachability, shortest distance, which are important for defining acquisition functions and feasible graphs. Encoding these graph properties into decision space is significantly more challenging because we must define constraints that prescribe all variables to have correct values for \textit{any} possible graph in the search space.

The graph encoding introduced in this paper incorporates reachability, shortest distances, and shortest paths for any graph without requiring strong connectivity as in previous work \citep{xie2024global}. These metrics are then used to encode shortest-path graph kernels for graph BO. To begin with, we define variables corresponding to relevant graph properties in Table \ref{tab:Var_full}. We consider all graphs with node number ranging from $n_0$ to $n$. For simplicity, we use $[n]$ to denote the set $\{0,1,\dots, n-1\}$.

For each variable $\mathit{Var}$ in Table \ref{tab:notations} , we use $\mathit{Var}(G)$ to denote its value on a given graph $G$. For example, $d_{u,v}(G)$ is the shortest distance from node $u$ to node $v$ in graph $G$. If graph $G$ is given, all variable values can be easily obtained using classic shortest-path algorithms, such as the Floyd–Warshall algorithm \citep{floyd1962algorithm}. However, for optimization over graphs, the variables must be constrained properly so that they take correct values for any given graph, i.e., to match the \texttt{Description} column in Table \ref{tab:notations}. Due to space limitations, we only present the final derived encoding in Eq.~\eqref{eq:final_MIP} and the major theory in Theorem \ref{thm:uniqueness}. Full derivations are given in Appendix \ref{app:shortest_path_encoding}.

Eq.~\eqref{eq:final_MIP} comprises many linear constraints resulting from Conditions $(\mathcal C1)$--$(\mathcal C8)$, as shown in Appendix \ref{app:encoding}. Here we present the final formulation, which conveys the overall idea about how to use constraints to mathematically define variables over graphs. 
Constraints for optimization formulations must be carefully selected. There are often multiple ways to encode a combinatorial problem, but insufficient constraints result in an unnecessarily large search space with symmetric solutions, while excessive constraints may cutoff feasible solutions from the search space. Theorem \ref{thm:uniqueness} guarantees that our encoding precisely formulates the graph space (see Appendix \ref{app:theory} for proofs).

\begin{theorem}\label{thm:uniqueness}
    There is a bijection between the feasible domain restricted by Eq.~\eqref{eq:final_MIP} with size $[n_0,n]$ and the whole graph space with node numbers in $[n_0,n]$.
\end{theorem}

The encoding proposed in \citep{xie2025bogrape} requires connected undirected graphs or strongly connected directed graphs, while our encoding Eq.~\eqref{eq:final_MIP} is more general and formulates the whole graph space without connectivity requirements. Observe that our encoding Eq.~\eqref{eq:final_MIP} can be easily restricted to the encoding in \citep{xie2025bogrape} by adding the following constraints:

\textbf{Undirected:} Add symmetry constraints to get undirected graphs:
\begin{equation*}
    \begin{aligned}
    A_{u,v}=A_{v,u},~r_{u,v}=r_{v,u},~d_{u,v}=d_{v,u},~\delta_{u,v}^w=\delta_{v,u}^w,~\forall u,v,w\in [n],~u<v.
    \end{aligned}
\end{equation*}

\textbf{Strong connectivity:} Each existing node can reach all other existing nodes, i.e.,
\begin{equation*}
    \begin{aligned}
        A_{u,u}=A_{v,v}=1\Rightarrow r_{u,v}=1,~\forall u,v\in [n],~u\neq v,
    \end{aligned}
\end{equation*}
which can be equivalently rewritten as the following constraint:
\begin{equation*}
    \begin{aligned}
        r_{u,v}\ge A_{u,u}+A_{v,v}-1,~\forall u,v\in [n],~u\neq v.
    \end{aligned}
\end{equation*}

\begin{remark}
    Note that strong connectivity reduces to connectivity for undirected graphs.
\end{remark}

\begin{table}[t]
    \centering
    \small
     \caption{Variables introduced to encode shortest paths for an arbitrary graph. Since the shortest distance between two nodes is always less than $n$, we use $n$ to denote infinity, i.e., $d_{u,v}=n$ means node $u$ cannot reach node $v$.}
     \label{tab:Var_full}
    \begin{tabular}{ccc}
        \toprule
        Variables & Domain & Description \\ 
        \midrule
        $A_{v,v},~v\in [n]$ & $\{0,1\}$ & if node $v$ exists\\
        $A_{u,v},~u,v\in [n],~u\neq v$ & $\{0,1\}$ &  if edge $u\to v$ exists \\
         $r_{u,v},~u,v\in [n]$ & $\{0,1\}$ & if node $u$ can reach node $v$\\
         $d_{u,v},~u,v\in [n]$ & $[n+1]$ &  the shortest distance from node $u$ to node $v$ \\
         $\delta_{u,v}^w,~u,v,w\in [n]$ & $\{0,1\}$ & if node $w$ appears on the shortest path from node $u$ to node $v$ \\
         \bottomrule
    \end{tabular}
    \label{tab:notations}
\end{table}

\begin{equation}\label{eq:final_MIP}\tag{Graph-Encoding}
    \left\{
    \begin{aligned}
        \sum\limits_{v\in [n]}A_{v,v}&\ge n_0&&\\
        A_{v,v}&\ge A_{v+1,v+1}\\
        2\cdot A_{u,v}&\le A_{u,u}+A_{v,v}\\
        2\cdot r_{u,v}&\le A_{u,u}+A_{v,v}\\
        d_{u,v}&\ge n\cdot (1-A_{u,u})\\
        d_{u,v}&\ge n\cdot (1-A_{v,v})\\
        r_{v,v}&=1\\
        d_{v,v}&=0\\
        \delta_{v,v}^v&=1\\
        \delta_{v,v}^w&=0\\
        r_{u,v}&\ge A_{u,v}\\
        d_{u,v}&\ge 2-A_{u,v}\\
        d_{u,v}&\le 1+(n-1)\cdot (1-A_{u,v})\\
        d_{u,v}&\le n-r_{u,v}\\
        d_{u,v}&\ge n-(n-1)\cdot r_{u,v}\\
        r_{u,w}+r_{w,v}&\ge 2\cdot \delta_{u,v}^w\\
        r_{u,v}&\ge r_{u,w}+r_{w,v}-1\\
        \delta_{u,v}^u&=\delta_{u,v}^v=1\\
        \sum\limits_{w\in [n]}\delta_{u,v}^w&\ge 2+r_{u,v}-A_{u,v}\\
        \sum\limits_{w\in [n]}\delta_{u,v}^w&\le 2+(n-2)\cdot(r_{u,v}-A_{u,v})\\
        d_{u,v}&\le d_{u,w}+d_{w,v}-(1-\delta_{u,v}^w)+(n+1)\cdot (2-r_{u,w}-r_{w,v})\\
        d_{u,v}&\ge d_{u,w}+d_{w,v}-2n\cdot (1-\delta_{u,v}^w)
    \end{aligned}
    \right.
\end{equation}  

\subsection{Restricting search space to the NAS domain}\label{subsec:NAS_domain}
The graphs considered in NAS are acyclic digraphs (DAGs), which are straightforward for formulate by adding the following constraints to Eq.~\eqref{eq:final_MIP}:
\begin{equation*}
    \begin{aligned}
        r_{u,v}+r_{v,u}\le 1,~\forall u,v\in [n],~u<v.
    \end{aligned}
\end{equation*}
In practice, however, the graph structures are more specific, e.g., having a single source (input) and single sink (output), and each graph may include node/edge labels. We consider DAGs with one source and one sink, which is the most classic setting in cell-based NAS. Based on the label type, we investigate two scenarios, i.e., node-labeled and edge-labeled DAGs, corresponding to the most commonly used benchmarks NAS-Bench-101 and NAS-Bench-201, resp.

\textbf{Node-labeled DAGs:} Following the NAS-Bench-101 \citep{ying2019bench} setting, we consider DAGs with $n$ nodes, at most $E$ edges, and $L_n$ different node labels (including two extra labels to identify the source and the sink). W.l.o.g., we use the first label for the source, and the last label for the sink. Introducing variable $F_{v,l}\in \{0,1\}$ to represent whether node $v\in [n]$ has label $l\in [L_n]$, we can write:
\begin{subequations}
    \begin{empheq}[left=\empheqlbrace]{align}
        &A_{u,v}=0,~r_{u,v}=0,~d_{u,v}=n,~\delta_{u,v}^w=0,~\forall u,v,w\in [n],~u>v,~w\neq u,v\label{eq:101_index}\\
        &r_{0,v}=1,~F_{0,0}=1,~F_{v,0}=0,~\forall v\in [n],~v\neq 0\label{eq:101_source}\\
        &r_{v,n-1}=1,~F_{n-1,L_n-1}=1,~F_{v,L_n-1}=0,~\forall v\in [n],~v\neq n-1\label{eq:101_sink}\\
        &\textstyle\sum_{l\in [L_n]}F_{v,l}=1,~\forall v\in [n]\label{eq:101_label}\\
        &\textstyle\sum_{u<v}A_{u,v}\le E\label{eq:101_edge}
    \end{empheq}
\end{subequations}
Eq.~\eqref{eq:101_index} enforces that each edge starts from the node with smaller index to reduce the number of isomorphic graphs. Eq.~\eqref{eq:101_source} sets node $0$ as the source, from which every other node can be reached. 
Eq.~\eqref{eq:101_sink} sets node $(n-1)$ as the sink, which every other node can reach. Eq.~\eqref{eq:101_label} enforces each node to take one label, and Eq.~\eqref{eq:101_edge} limits the maximal number of edges.

\begin{remark}
    NAS-Bench-101 is a particular instance of the above, i.e., with $n=7,~E=9,~L_n=5$.
\end{remark}

\textbf{Edge-labeled DAGs:} Following the NAS-Bench-201 \citep{dong2020bench} and DARTS settings, we consider DAGs with $n$ nodes (meaning all nodes are already indexed) and $L_e$ edge labels. Introducing variable $F_{u\to v,l}$ to represent whether edge $u\to v$ (with $u< v$) has label $l\in [L_e]$, we have the encoding:
\begin{subequations}
    \begin{empheq}[left=\empheqlbrace]{align}
        &A_{u,v}=0,~r_{u,v}=0,~d_{u,v}=n,~\delta_{u,v}^w=0,~\forall u,v,w\in [n],~u>v,~w\neq u,v\label{eq:201_index}\\
        &r_{0,v}=r_{v,n-1}=1,~\forall v\in [n]\label{eq:201_source}\\
        &\textstyle\sum_{l\in [L_e]}F_{u\to v,l}=A_{u,v},~\forall u,v\in [n],~u<v\label{eq:201_label}
    \end{empheq}
\end{subequations}
Eq.~\eqref{eq:201_index} is the same as Eq.~\eqref{eq:101_index}, Eq.~\eqref{eq:201_label} sets node $0$ as the source  and node $(n-1)$ as the sink, and Eq.~\eqref{eq:201_label} forces one edge label for each existing edge and no edge labels for nonexistent edges.

\begin{remark}
    NAS-Bench-201 is a particular instance of the above, i.e., with $n=4,~L_e=4$. Note that NAS-Bench-201 has $5$ labels: one label denotes  nonexistance, which is not needed in our encoding.
\end{remark}

\subsection{Encode graph kernels}\label{subsec:kernel_encoding}
We take the triple $(G,F_n,F_e)$ as a graph with node labels $F_n=\{F_{v,l}\}_{v\in [n],~l\in [L_n]}$ and edge labels $F_e=\{F_{u\to v,l}\}_{u,v\in [n],~l\in [L_e]}$. Given two labeled graphs $X^1=(G^1,F_n^1,F_e^1)$ and $X^2=(G^2,F_n^2,F_e^2)$ and denoting their node numbers as $n_1$ and $n_2$, resp., we define the following general kernel form:
\begin{equation}\label{eq:linear_kernel}\tag{linear}
    \begin{aligned}
        k_{\text{lin}}(X^1,X^2)=\alpha \cdot k_g(G^1,G^2)+\beta\cdot k_n(F_n^1,F_n^2)+\gamma\cdot k_e(F_e^1,F_e^2),
    \end{aligned}
\end{equation}
where kernels $k_g,k_n,k_e$ quantify similarity over graph structure, node labels, and edge labels, resp.

We then denote the optimization target as an unknown graph $x=(G,F_n,F_e)$, and the available data points as $X=\{X^i,y^i\}_{i=1}^{t-1}$ with $X^i=(G^i,F_n^i,F_e^i)$. After properly defining the search space in Section \ref{subsec:NAS_domain}, the last step is to encode kernel-relevant terms in Eq.~\eqref{eq:BoGrape}, i.e., $k_{xX^i}$ and $k_{xx}$. For graph structure $G$ and node labels $F_n$, the encoding of the graph structure \eqref{eq:SP_kernel} kernel, i.e., $k_{g}(G,G),k_{g}(G,G^i)$ and binary node features, i.e., $k_n(F_n,F_n),k_n(F_n,F_n^i)$ are given in \citep{xie2025bogrape}. For completeness, we provide details in Appendix \ref{app:kernel_encoding}. 

\textbf{Edge label encoding:} Edge labels can be treated in a similar way to node labels. However, several NAS settings have more specific properties, i.e., all nodes are indexed when edge labels are present, and all graphs have the same size. Thus we alternatively propose the following simplified form:
\begin{equation}\label{eq:kernel_edge}\tag{$k_e$}
    \begin{aligned}
        k_e(F_e^1,F_e^2)=\frac{2}{n(n-1)}\langle F_e^1, F_e^2\rangle=\frac{2}{n(n-1)}\sum\limits_{u<v}\sum\limits_{l\in [L_e]}F_{u\to v,l}^1\cdot F_{u\to v,l}^2,
    \end{aligned}
\end{equation}
where $n(n-1)/2$ is a normalizing coefficient, with $n$ as the node number of both $G^1$ and $G^2$, given that a DAG has at most $n(n-1)/2$ edges. 

We take edge kernels as follows, and evaluate their performance in Section~\ref{subsec:kernel_comparison}:
\begin{equation*}
    \begin{aligned}
        k_e(F_e,F_e^i)&=\frac{2}{n(n-1)}\sum\limits_{u<v}\sum\limits_{l\in [L_e]}F_{u\to v,l}^i\cdot F_{u\to v,l},\\
        k_e(F_e,F_e)&=\frac{2}{n(n-1)}\sum\limits_{u<v}\sum\limits_{l\in [L_e]}F_{u\to v,l}^2=\frac{2}{n(n-1)}\sum\limits_{u<v}\sum\limits_{l\in [L_e]}F_{u\to v,l}=\frac{2}{n(n-1)}\sum\limits_{u<v}A_{u,v},
    \end{aligned}
\end{equation*}
where we use the trick that $x^2=x$ for binary $x$ and the relation in Eq.~\eqref{eq:201_label}.

The above defines a formulation for all relevant terms in kernel form \eqref{eq:linear_kernel}. To improve representation ability, we also consider an alternative exponential form defined as~\cite{xie2025bogrape}:
\begin{equation}\label{eq:exponential_kernel}\tag{exponential}
    \begin{aligned}
        k_{\exp}(X^1,X^2)=\sigma_k^2\cdot \exp(k_{\text{lin}}(X^1,X^2)),
    \end{aligned}
\end{equation}
where the variance $\sigma_k^2$ controls the magnitude of kernel values.

\section{Experiments}\label{sec:experiments}
\begin{table}[t]
    \caption{NAS algorithms comparison, including whether the method is BO-based, the surrogate model used, and how acquisition function is optimized (if BO-based). The superscript `$^a$' denotes methods that are not originally designed for NAS but can be adapted for NAS settings. For surrogate models, we use `v' to denote models using vectorized embeddings of graphs and `g' to denote models that directly over graph spaces.}
    \label{table:baselines}
    \begin{center}
    \begin{tabular}{cccc}
        \toprule
        Algorithms & BO-based & Surrogate & Acquisition optimization\\
        \midrule
        Random & $\times$ & - & - \\
        DNGO$^a$~\citep{snoek2015dngo} & \checkmark & BNN(v) & mutation \\
        BOHAMIANN$^a$~\citep{springenberg2016bohamiann} & \checkmark & BNN(v) & mutation \\
        NASBOT~\citep{kandasamy2018transport} & \checkmark & GP(g) & mutation \\
        Evolution~\citep{real2019evolution} & $\times$ & - & - \\
        GP-BAYESOPT$^a$~\citep{neiswanger2019probo} & \checkmark & GP(v) & sampling \\
        GCN~\citep{wen2020gcn} & $\times$ & - & - \\
        BONAS~\citep{shi2020bonas} & \checkmark & GCN(v) & sampling \\
        Local search~\citep{white2021localsearch} & $\times$ & - & - \\
        BANANAS~\citep{white2021bananas} & \checkmark & NN(v) & mutation \\
        NAS-BOWL~\citep{ru2021interpretable} & \checkmark & GP(g) & mutation \\
        NAS-GOAT (ours) & \checkmark & GP(g) & MIP \\
     \bottomrule
    \end{tabular}
    \end{center}
\end{table} 

All experiments are performed on a 4.2 GHz Intel Core i7-7700K CPU with 16 GB memory. For our methods, we use GPflow \citep{matthews2017GPflow} to implement GP models, and Gurobi \citep{gurobi2024} to solve MIPs. For kernel comparison, GraKel \citep{siglidis2020grakel} is used to implement graph kernels. We use the published implementations of NAS-BOWL \citep{ru2021interpretable} and Naszilla \citep{white2020study,white2021localsearch,white2021bananas} for all other NAS baselines. 

\subsection{Benchmarks}\label{subsec:benchmarks}

We choose the most popular benchmarks used in NAS literature, i.e., NAS-Bench-101 \citep{ying2019bench} and NAS-Bench-201 \citep{dong2020bench}, to evaluate the performance of our proposed method. These two benchmarks correspond to the node- and edge-labeled cases as discussed in Section \ref{subsec:NAS_domain}, resp..

\textbf{NAS-Bench-101:} DAGs with one source, one sink, at most 7 nodes and 9 edges, and 3 different node operations. Only the source is labeled as operation \texttt{IN}, only the sink is labeled as operation \texttt{OUT}, and each of other nodes has one of the remaining three operations: 3x3 convolution, 1x1 convolution, or 3x3 max pooling. After removal of duplicates, NAS-Bench-101 has approximately 423k unique architectures. Each architecture is trained on CIFAR-10 to obtain validation and test accuracies.

\textbf{NAS-Bench-201:} Dense DAGs with 4 nodes. Each of the 6 edges has a label chosen from 5 operation types: zeroize, skip-connection, 1x1 convolution, 3x3 convolution, or 3x3 average pooling. NAS-Bench-201 has 15,625 architectures in total, each of which has various metrics including validation and test accuracies over three datasets: CIFAR10, CIFAR100, and ImageNet-16-120. 

In both benchmarks, each architecture is trained 20 times with varying random seeds, which could be used as a noisy objective function as suggested in \citep{ru2021interpretable}. 
We conduct experiments for both scenarios, reporting results for the deterministic setting here, i.e., averaging the accuracies over multiple random seeds. Results for noisy setting are given in Appendix \ref{app:full_experiments}. 

\begin{figure}[t]
    \centering
    \includegraphics[width=\linewidth]{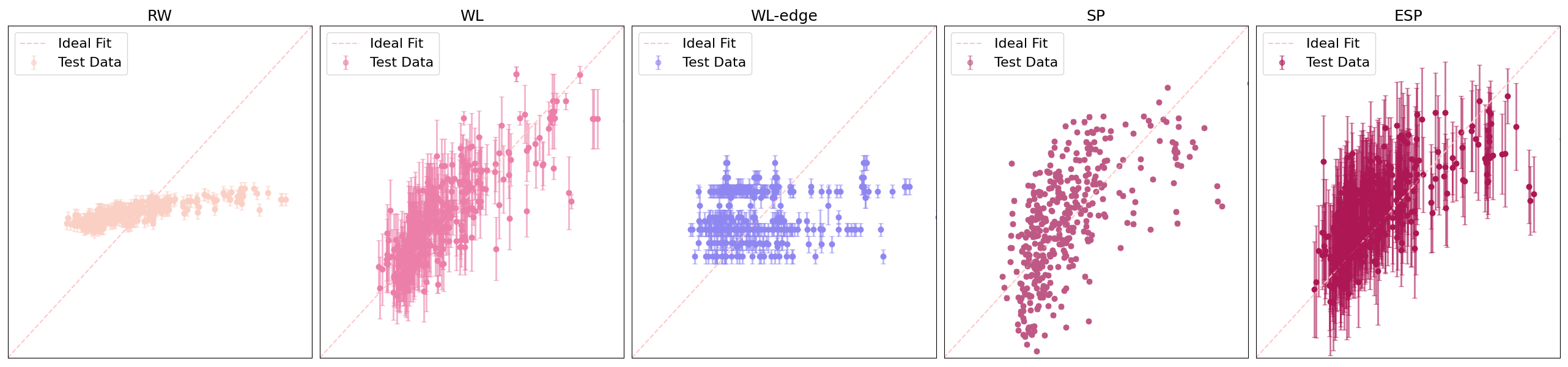}
    \caption{Predictive performance of graph GPs with different kernels. 50 and 400 architectures are randomly sampled from NAS-Bench-201 for training and testing, resp. Predicted deterministic validation error are plotted against the true values, with one standard deviation as error bars.}
    \label{fig:N201_kernel_compare}
\end{figure}

\begin{table}[t]
    \caption{GP model performance metrics using different graph kernels. For each dataset, 50 and 400 architectures are sampled for training and testing, resp. Predictive performance metrics are averaged over 20 replications and reported in the table, with one standard deviation in the brackets. The best method is marked in \textbf{bold} metric-wisely for each dataset.}
    \label{table:model_performance}
    \begin{center}
    \begin{small}
    \begin{tabular}{ccccccc}
        \toprule
        Dataset & \multicolumn{3}{c}{NAS-Bench-101} & \multicolumn{3}{c}{NAS-Bench-201}\\
        \midrule
        Kernel & RMSE $\downarrow$ & MNLL $\downarrow$ & Spearman $\uparrow$ & RMSE $\downarrow$ & MNLL $\downarrow$ & Spearman $\uparrow$\\
        \midrule
        RW & 0.29(0.01) & 30.29(5.26) & 0.81(0.04) & 0.32(0.02) & 43.07(19.14) & 0.78(0.05)\\
        WL & 0.15(0.02) & \textbf{-0.77(0.07)} & 0.87(0.03) & \textbf{0.23(0.04)} & 0.98(1.02) & \textbf{0.81(0.07)}\\
        WL-edge & - & - & - & 0.37(0.02) & 46.00(16.73) & 0.11(0.09)\\
        SP & 0.21(0.05) & 227.67(114.59) & 0.83(0.04) & 0.33(0.04) & 465.15(152.20) & 0.63(0.10)\\
        ESP & \textbf{0.11(0.02)} & 28.83(15.02) & \textbf{0.93(0.02)} & 0.30(0.04) & \textbf{0.36(0.22)} & 0.64(0.11)\\
     \bottomrule
    \end{tabular}
    \end{small}
    \end{center}
\end{table} 

\subsection{Baselines}\label{subsec:baselines}
We compare our method, NAS-GOAT, which is capable of \textbf{g}lobally \textbf{o}ptimizing \textbf{a}cquisi\textbf{t}ion in form \eqref{eq:BoGrape} with the encoding introduced in Section \ref{sec:methodology}, against state-of-the-art baselines in NAS, summarized in Table~\ref{table:baselines}. BO-based baselines either use GPs or neural predictors as the surrogate model. Graph inputs are featurized into vectors using different encoding methods before feeding to NN surrogates~\citep{snoek2015dngo, springenberg2016bohamiann, shi2020bonas, white2021bananas}. For GP surrogate models, graphs can be directly used as data points by defining a proper graph kernel~\citep{ru2021interpretable} or graph similarity metric~\citep{kandasamy2018transport}. In addition to BO-based algorithms, we also include classic methods in NAS such as random search (Random), regularized evolution (Evolution)~\citep{real2019evolution}, local search (Local seach)~\citep{white2021localsearch} and GCN predictor (GCN)~\citep{wen2020gcn}. For BO-based methods, optimization of acquisition functions is achieved through mutation or sampling, while NAS-GOAT is capable of globally acquisition optimization over graph search spaces. More descriptions and implementation details of the baselines can be found in Appendix~\ref{app:full_experiments}.

\subsection{Graph kernels comparison}\label{subsec:kernel_comparison}
Although the focus of this paper is global acquisition optimization rather than graph modeling, the kernel performance is still important to overall BO performance, noting that we employ SP kernels. In this section, we compare the predictive performance of graph GPs equipped with various graph kernels. For NAS-Bench-101 with node labels, we compare RW, WL, and our kernels in form \eqref{eq:linear_kernel} (SP) and \eqref{eq:exponential_kernel} (ESP). For NAS-Bench-201 with edge labels, all architectures are first converted to node-labeled graphs and then evaluated over RW and WL kernels. We also test the performance of WL kernels over the original edge-labeled graphs (denoted as WL-e). Both SP and ESP kernels can directly handle edge labels without conversion. 

Figure \ref{fig:N201_kernel_compare} illustrates the predictive performance of different kernels on NAS-Bench-201 (see Figure \ref{fig:N101_kernel_compare} in Appendix \ref{app:full_experiments} for a similar plot for NAS-Bench-101), and Table \ref{table:model_performance} reports performance metrics including root mean squared error (RMSE) and Spearman's rank correlation (Spearman) showing the predictive accuracy, as well as mean negative log likelihood (MNLL) measuring the uncertainty qualification. The RW kernel does not perform well on both cases. The WL kernel performs significantly better on converted node-labeled graphs compared to the original edge-labeled graphs, which matches the empirical observations in \citep{ru2021interpretable}. WL and ESP kernels have comparably good performance, both of which outperform the SP kernel. Note that a better kernel may not necessarily translate to better optimization results, since complex kernel forms bring extra difficulties in acquisition optimization, resulting in a computational trade-off between the modeling and optimization steps.  

\subsection{Graph BO for NAS}\label{subsec:BO_NAS}
In this section, we test the performance of NAS-GOAT in a full BO loop over NAS benchmarks against baselines. Following the batch setting in \cite{white2021bananas,ru2021interpretable}, we conduct $30$ BO iterations starting with 10 initial samples. At each iteration, we solve the MIP defined by Eq.~\eqref{eq:final_MIP} using Gurobi \cite{gurobi2024} and store the best 5 candidates (in terms of acquisition function value) to evaluate.

We denote our methods as NAS-GOAT-L (using kernel \eqref{eq:linear_kernel}) and NAS-GOAT-E (using kernel \eqref{eq:exponential_kernel}) to differentiate the kernel used in graph GP. All baselines introduced in Table \ref{table:baselines} are implemented, but we only report Random, GCN~\citep{wen2020gcn}, Evolution~\citep{real2019evolution}, NASBOT~\citep{kandasamy2018transport}, BANANAS~\citep{white2021bananas} and NAS-BOWL~\citep{ru2021interpretable} here for conciseness, since they usually achieve better results. Full baselines are given in Appendix \ref{app:full_experiments}. Following NAS literature, we minimize over validation error and report both validation and test errors. As shown in Figure \ref{fig:deterministic_main_paper}, both NAS-GOAT-L and NAS-GOAT-E find (near-)optimal architectures in all scenarios. NAS-GOAT-L achieves slightly better performance perhaps owing to its simpler form, making the resulting optimization formulation less complicated.

\begin{figure}
     \centering
     \includegraphics[width=\textwidth]{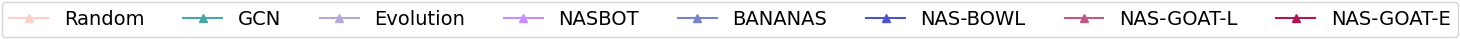}\\
     \vspace{.3mm}
     \begin{subfigure}[b]{0.245\textwidth}
         \centering
         \includegraphics[width=\textwidth]{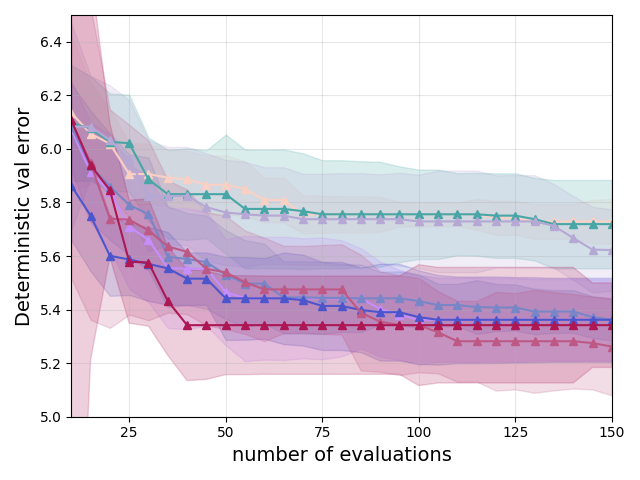}
         \caption{N101 (CIFAR10)}
     \end{subfigure}
     \hfill
    \begin{subfigure}[b]{0.245\textwidth}
         \centering
         \includegraphics[width=\textwidth]{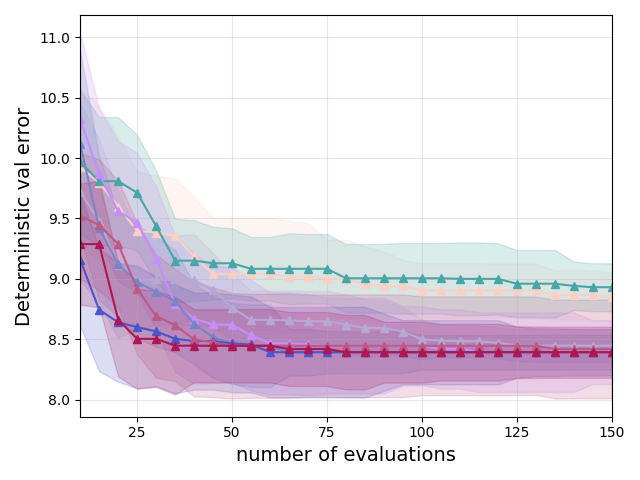}
         \caption{N201 (CIFAR10)}
     \end{subfigure}
     \hfill
     \begin{subfigure}[b]{0.245\textwidth}
         \centering
         \includegraphics[width=\textwidth]{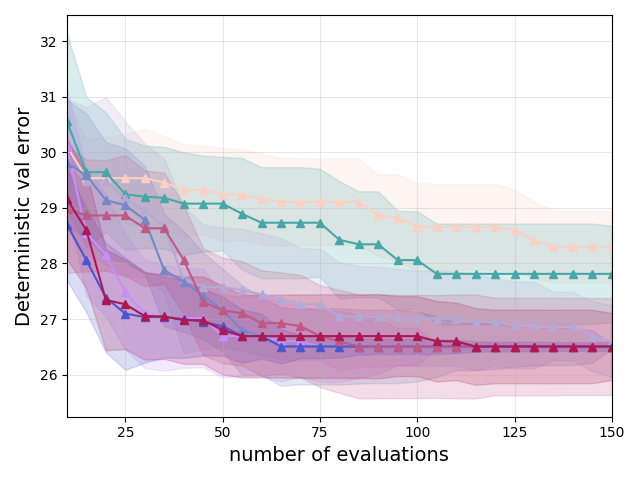}
         \caption{N201 (CIFAR100)}
     \end{subfigure}
    \hfill
     \begin{subfigure}[b]{0.245\textwidth}
         \centering
         \includegraphics[width=\textwidth]{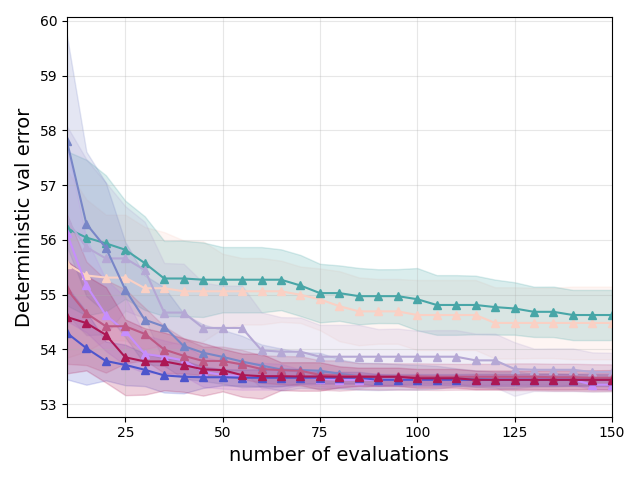}
         \caption{N201 (ImageNet)}
     \end{subfigure}
     \hfill
     \begin{subfigure}[b]{0.245\textwidth}
         \centering
         \includegraphics[width=\textwidth]{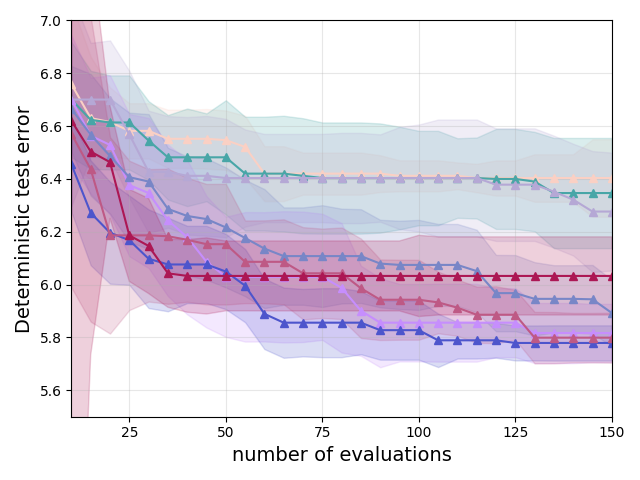}
         \caption{N101 (CIFAR10)}
     \end{subfigure}
     \hfill
     \begin{subfigure}[b]{0.245\textwidth}
         \centering
         \includegraphics[width=\textwidth]{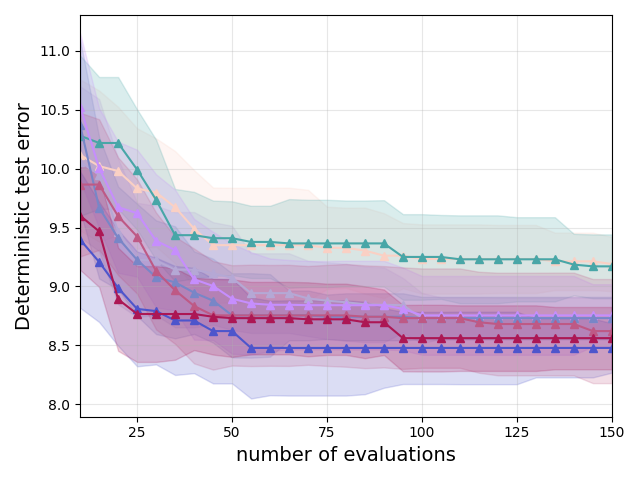}
         \caption{N201 (CIFAR10)}
     \end{subfigure}
     \hfill
     \begin{subfigure}[b]{0.245\textwidth}
         \centering
         \includegraphics[width=\textwidth]{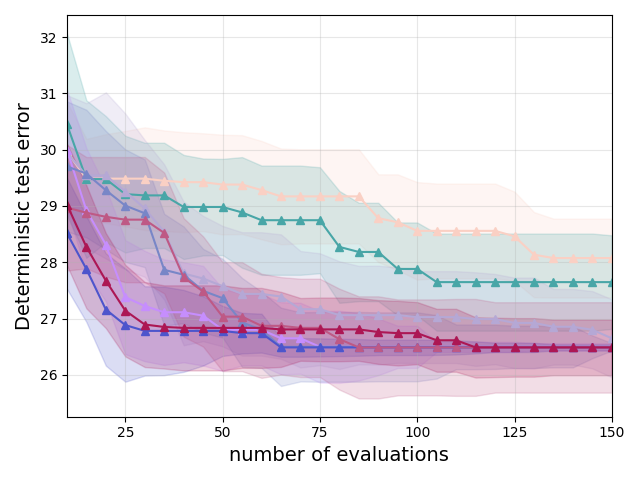}
         \caption{N201 (CIFAR100)}
     \end{subfigure}
     \hfill
     \begin{subfigure}[b]{0.245\textwidth}
         \centering
         \includegraphics[width=\textwidth]{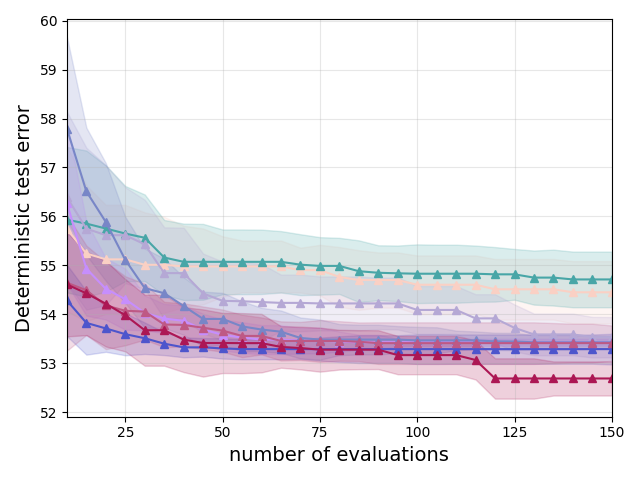}
         \caption{N201 (ImageNet)}
     \end{subfigure}
    \caption{Numerical results of Graph BO on NAS-Bench-101 (N101) and NAS-Bench-201 (N201). (\textbf{Top}) Deterministic validation error. (\textbf{Bottom}) The corresponding test error. Median with one standard deviation over 20 replications is plotted.}
    \label{fig:deterministic_main_paper}
\end{figure}

\section{Conclusion}\label{sec:conclusion}
This work considers global acquisition optimization in graph BO for NAS. The graph search space is precisely encoded into an equivalent variable space for discrete optimization. A general kernel is designed to handle both node and edge labels, and formulations are proposed based on the graph encoding. After adding suitable constraints to remove invalid architectures, we are able to globally optimize the acquisition function at each BO iteration, demonstrating promising results on commonly used NAS benchmarks. Future works could consider more graph kernels beyond shortest-path kernels, or apply the proposed method to more graph-based decision making problems. 

\section*{Acknowledgments}
The authors gratefully acknowledge support from a Department of Computing Scholarship (YX), BASF SE, Ludwigshafen am Rhein (SZ), Engineering and Physical Sciences Research Council [grant numbers EP/W003317/1 and EP/X025292/1] (RM, CT, JQ), a BASF/RAEng Research Chair in Data-Driven Optimisation (RM), a BASF/RAEng Senior Research Fellowship (CT). RM holds concurrent appointments as a Professor at Imperial and as an Amazon Scholar. This paper describes work performed at Imperial prior to joining Amazon and is not associated with Amazon.

\bibliographystyle{abbrvnat}
\bibliography{ref}

\begin{thebibliography}{60}
\providecommand{\natexlab}[1]{#1}
\providecommand{\url}[1]{\texttt{#1}}
\expandafter\ifx\csname urlstyle\endcsname\relax
  \providecommand{\doi}[1]{doi: #1}\else
  \providecommand{\doi}{doi: \begingroup \urlstyle{rm}\Url}\fi

\bibitem[Ammari et~al.(2023)Ammari, Johnson, Stinchfield, Kim, Bynum, Hart, Pulsipher, and Laird]{ammari2023linear}
B.~L. Ammari, E.~S. Johnson, G.~Stinchfield, T.~Kim, M.~Bynum, W.~E. Hart, J.~Pulsipher, and C.~D. Laird.
\newblock Linear model decision trees as surrogates in optimization of engineering applications.
\newblock \emph{Computers \& Chemical Engineering}, 178, 2023.

\bibitem[Anderson et~al.(2020)Anderson, Huchette, Ma, Tjandraatmadja, and Vielma]{anderson2020strong}
R.~Anderson, J.~Huchette, W.~Ma, C.~Tjandraatmadja, and J.~P. Vielma.
\newblock Strong mixed-integer programming formulations for trained neural networks.
\newblock \emph{Mathematical Programming}, 183\penalty0 (1):\penalty0 3--39, 2020.

\bibitem[Borgwardt et~al.(2020)Borgwardt, Ghisu, Llinares-L{\'o}pez, O’Bray, Rieck, et~al.]{borgwardt2020graph}
K.~Borgwardt, E.~Ghisu, F.~Llinares-L{\'o}pez, L.~O’Bray, B.~Rieck, et~al.
\newblock Graph kernels: State-of-the-art and future challenges.
\newblock \emph{Foundations and Trends{\textregistered} in Machine Learning}, 13\penalty0 (5-6):\penalty0 531--712, 2020.

\bibitem[Borgwardt and Kriegel(2005)]{borgwardt2005shortest}
K.~M. Borgwardt and H.-P. Kriegel.
\newblock Shortest-path kernels on graphs.
\newblock In \emph{International Conference on Data Mining}, 2005.

\bibitem[Cheng et~al.(2022)Cheng, Wang, Zhang, Chen, Wang, and Cheng]{cheng2022dpnas}
A.~Cheng, J.~Wang, X.~Zhang, Q.~Chen, P.~Wang, and J.~Cheng.
\newblock {DPNAS}: Neural architecture search for deep learning with differential privacy.
\newblock \emph{AAAI}, 2022.

\bibitem[Dong and Yang(2020)]{dong2020bench}
X.~Dong and Y.~Yang.
\newblock {NAS-Bench-201}: Extending the scope of reproducible neural architecture search.
\newblock In \emph{ICLR}, 2020.

\bibitem[Elsken et~al.(2019)Elsken, Metzen, and Hutter]{elsken2019NASsurvey}
T.~Elsken, J.~H. Metzen, and F.~Hutter.
\newblock Neural architecture search: a survey.
\newblock \emph{Journal of Machine Learning Research}, 2019.

\bibitem[Fischetti and Jo(2018)]{fischetti2018deep}
M.~Fischetti and J.~Jo.
\newblock Deep neural networks and mixed integer linear optimization.
\newblock \emph{Constraints}, 23\penalty0 (3):\penalty0 296--309, 2018.

\bibitem[Floyd(1962)]{floyd1962algorithm}
R.~W. Floyd.
\newblock Algorithm 97: Shortest path.
\newblock \emph{Communications of the ACM}, 5\penalty0 (6):\penalty0 345--345, 1962.

\bibitem[Frazier(2018)]{frazier2018BO}
P.~I. Frazier.
\newblock A tutorial on {B}ayesian optimization.
\newblock \emph{arXiv preprint arXiv:1807.02811}, 2018.

\bibitem[Garnett(2023)]{garnett2023BObook}
R.~Garnett.
\newblock \emph{Bayesian Optimization}.
\newblock Cambridge University Press, 2023.

\bibitem[G{\"a}rtner et~al.(2003)G{\"a}rtner, Flach, and Wrobel]{gartner2003graph}
T.~G{\"a}rtner, P.~Flach, and S.~Wrobel.
\newblock On graph kernels: Hardness results and efficient alternatives.
\newblock In \emph{Learning Theory and Kernel Machines}, 2003.

\bibitem[{Gurobi Optimization, LLC}(2024)]{gurobi2024}
{Gurobi Optimization, LLC}.
\newblock {Gurobi optimizer reference manual}, 2024.
\newblock URL \url{https://www.gurobi.com}.

\bibitem[Hojny et~al.(2024)Hojny, Zhang, Campos, and Misener]{hojny2024verifying}
C.~Hojny, S.~Zhang, J.~S. Campos, and R.~Misener.
\newblock Verifying message-passing neural networks via topology-based bounds tightening.
\newblock In \emph{ICML}, 2024.

\bibitem[Huchette et~al.(2023)Huchette, Mu{\~n}oz, Serra, and Tsay]{huchette2023deep}
J.~Huchette, G.~Mu{\~n}oz, T.~Serra, and C.~Tsay.
\newblock When deep learning meets polyhedral theory: A survey.
\newblock \emph{arXiv preprint arXiv:2305.00241}, 2023.

\bibitem[Jaafra et~al.(2019)Jaafra, Laurent, Deruyver, and Naceur]{jaafra2019rlreview}
Y.~Jaafra, J.~L. Laurent, A.~Deruyver, and M.~S. Naceur.
\newblock Reinforcement learning for neural architecture search: A review.
\newblock \emph{Image and Vision Computing}, 89:\penalty0 57--66, 2019.

\bibitem[Kandasamy et~al.(2018)Kandasamy, Neiswanger, Schneider, Poczos, and Xing]{kandasamy2018transport}
K.~Kandasamy, W.~Neiswanger, J.~Schneider, B.~Poczos, and E.~P. Xing.
\newblock Neural architecture search with {B}ayesian optimisation and optimal transport.
\newblock \emph{NeurIPS}, 31, 2018.

\bibitem[Kriege and Mutzel(2012)]{kriege2012subgraph}
N.~Kriege and P.~Mutzel.
\newblock Subgraph matching kernels for attributed graphs.
\newblock In \emph{ICML}, 2012.

\bibitem[Kriege et~al.(2016)Kriege, Giscard, and Wilson]{kriege2016valid}
N.~M. Kriege, P.-L. Giscard, and R.~Wilson.
\newblock On valid optimal assignment kernels and applications to graph classification.
\newblock \emph{NeurIPS}, 2016.

\bibitem[Kriege et~al.(2020)Kriege, Johansson, and Morris]{kriege2020survey}
N.~M. Kriege, F.~D. Johansson, and C.~Morris.
\newblock A survey on graph kernels.
\newblock \emph{Applied Network Science}, 5:\penalty0 1--42, 2020.

\bibitem[Liu et~al.(2018)Liu, Zoph, Neumann, Shlens, Hua, Li, Fei-Fei, Yuille, Huang, and Murphy]{liu2018progressive}
C.~Liu, B.~Zoph, M.~Neumann, J.~Shlens, W.~Hua, L.-J. Li, L.~Fei-Fei, A.~Yuille, J.~Huang, and K.~Murphy.
\newblock Progressive neural architecture search.
\newblock In \emph{ECCV}, 2018.

\bibitem[Liu et~al.(2019)Liu, Simonyan, and Yang]{liu2019darts}
H.~Liu, K.~Simonyan, and Y.~Yang.
\newblock {DARTS}: Differentiable architecture search.
\newblock In \emph{ICLR}, 2019.

\bibitem[Matthews et~al.(2017)Matthews, {van der Wilk}, Nickson, Fujii, {Boukouvalas}, {Le{\'o}n-Villagr{\'a}}, Ghahramani, and Hensman]{matthews2017GPflow}
A.~G. d.~G. Matthews, M.~{van der Wilk}, T.~Nickson, K.~Fujii, A.~{Boukouvalas}, P.~{Le{\'o}n-Villagr{\'a}}, Z.~Ghahramani, and J.~Hensman.
\newblock {{GP}flow: A {G}aussian process library using {T}ensor{F}low}.
\newblock \emph{Journal of Machine Learning Research}, 18\penalty0 (40):\penalty0 1--6, 2017.

\bibitem[McDonald et~al.(2024)McDonald, Tsay, Schweidtmann, and Yorke-Smith]{mcdonald2024mixed}
T.~McDonald, C.~Tsay, A.~M. Schweidtmann, and N.~Yorke-Smith.
\newblock Mixed-integer optimisation of graph neural networks for computer-aided molecular design.
\newblock \emph{Computers \& Chemical Engineering}, 185:\penalty0 108660, 2024.

\bibitem[Mi{\v{s}}i{\'c}(2020)]{misic2020optimization}
V.~V. Mi{\v{s}}i{\'c}.
\newblock Optimization of tree ensembles.
\newblock \emph{Operations Research}, 68\penalty0 (5):\penalty0 1605--1624, 2020.

\bibitem[Mistry et~al.(2021)Mistry, Letsios, Krennrich, Lee, and Misener]{mistry2021mixed}
M.~Mistry, D.~Letsios, G.~Krennrich, R.~M. Lee, and R.~Misener.
\newblock Mixed-integer convex nonlinear optimization with gradient-boosted trees embedded.
\newblock \emph{INFORMS Journal on Computing}, 33\penalty0 (3):\penalty0 1103--1119, 2021.

\bibitem[Neiswanger et~al.(2019)Neiswanger, Kandasamy, Poczos, Schneider, and Xing]{neiswanger2019probo}
W.~Neiswanger, K.~Kandasamy, B.~Poczos, J.~Schneider, and E.~Xing.
\newblock {ProBo}: a framework for using probabilistic programming in {B}ayesian optimization.
\newblock \emph{arXiv preprint arXiv:1901.11515}, 2019.

\bibitem[Nikolentzos et~al.(2021)Nikolentzos, Siglidis, and Vazirgiannis]{nikolentzos2021graph}
G.~Nikolentzos, G.~Siglidis, and M.~Vazirgiannis.
\newblock Graph kernels: A survey.
\newblock \emph{Journal of Artificial Intelligence Research}, 72:\penalty0 943--1027, 2021.

\bibitem[Qiu et~al.(2023)Qiu, Bi, Xu, Guo, Ge, Liang, Lee, and Wu]{qiu2023selfevol}
Z.~Qiu, W.~Bi, D.~Xu, H.~Guo, H.~Ge, Y.~Liang, H.~P. Lee, and C.~Wu.
\newblock Efficient self-learning evolutionary neural architecture search.
\newblock \emph{Applied Soft Computing}, 146:\penalty0 110671, 2023.

\bibitem[Real et~al.(2019)Real, Aggarwal, Huang, and Le]{real2019evolution}
E.~Real, A.~Aggarwal, Y.~Huang, and Q.~V. Le.
\newblock Regularized evolution for image classifier architecture search.
\newblock In \emph{AAAI}, 2019.

\bibitem[Ren et~al.(2021)Ren, Xiao, Chang, Huang, Li, Chen, and Wang]{ren2021comprehensive}
P.~Ren, Y.~Xiao, X.~Chang, P.-Y. Huang, Z.~Li, X.~Chen, and X.~Wang.
\newblock A comprehensive survey of neural architecture search: Challenges and solutions.
\newblock \emph{ACM Computing Surveys}, 54\penalty0 (4):\penalty0 1--34, 2021.

\bibitem[Ru et~al.(2021)Ru, Wan, Dong, and Osborne]{ru2021interpretable}
B.~Ru, X.~Wan, X.~Dong, and M.~Osborne.
\newblock Interpretable neural architecture search via {B}ayesian optimisation with {W}eisfeiler-{L}ehman kernels.
\newblock In \emph{ICLR}, 2021.

\bibitem[Salmani Pour~Avval et~al.(2025)Salmani Pour~Avval, Eskue, Groves, and Yaghoubi]{salmani2025systematic}
S.~Salmani Pour~Avval, N.~D. Eskue, R.~M. Groves, and V.~Yaghoubi.
\newblock Systematic review on neural architecture search.
\newblock \emph{Artificial Intelligence Review}, 58\penalty0 (3):\penalty0 73, 2025.

\bibitem[Schulz et~al.(2018)Schulz, Speekenbrink, and Krause]{schulz2018GP}
E.~Schulz, M.~Speekenbrink, and A.~Krause.
\newblock A tutorial on {G}aussian process regression: Modelling, exploring, and exploiting functions.
\newblock \emph{Journal of mathematical psychology}, 85, 2018.

\bibitem[Schweidtmann et~al.(2021)Schweidtmann, Bongartz, Grothe, Kerkenhoff, Lin, Najman, and Mitsos]{schweidtmann2021deterministic}
A.~M. Schweidtmann, D.~Bongartz, D.~Grothe, T.~Kerkenhoff, X.~Lin, J.~Najman, and A.~Mitsos.
\newblock Deterministic global optimization with {G}aussian processes embedded.
\newblock \emph{Mathematical Programming Computation}, 13\penalty0 (3):\penalty0 553--581, 2021.

\bibitem[Shervashidze et~al.(2011)Shervashidze, Schweitzer, Van~Leeuwen, Mehlhorn, and Borgwardt]{shervashidze2011weisfeiler}
N.~Shervashidze, P.~Schweitzer, E.~J. Van~Leeuwen, K.~Mehlhorn, and K.~M. Borgwardt.
\newblock Weisfeiler-{l}ehman graph kernels.
\newblock \emph{Journal of Machine Learning Research}, 12\penalty0 (9), 2011.

\bibitem[Shi et~al.(2020)Shi, Pi, Xu, Li, Kwok, and Zhang]{shi2020bonas}
H.~Shi, R.~Pi, H.~Xu, Z.~Li, J.~Kwok, and T.~Zhang.
\newblock Bridging the gap between sample-based and one-shot neural architecture search with {BONAS}.
\newblock \emph{NeurIPS}, 2020.

\bibitem[Siglidis et~al.(2020)Siglidis, Nikolentzos, Limnios, Giatsidis, Skianis, and Vazirgiannis]{siglidis2020grakel}
G.~Siglidis, G.~Nikolentzos, S.~Limnios, C.~Giatsidis, K.~Skianis, and M.~Vazirgiannis.
\newblock Gra{K}el: A graph kernel library in {P}ython.
\newblock \emph{Journal of Machine Learning Research}, 21\penalty0 (54):\penalty0 1--5, 2020.

\bibitem[Snoek et~al.(2015)Snoek, Rippel, Swersky, Kiros, Satish, Sundaram, Patwary, Prabhat, and Adams]{snoek2015dngo}
J.~Snoek, O.~Rippel, K.~Swersky, R.~Kiros, N.~Satish, N.~Sundaram, M.~Patwary, M.~Prabhat, and R.~Adams.
\newblock Scalable {B}ayesian optimization using deep neural networks.
\newblock In \emph{ICML}, 2015.

\bibitem[Springenberg et~al.(2016)Springenberg, Klein, Falkner, and Hutter]{springenberg2016bohamiann}
J.~T. Springenberg, A.~Klein, S.~Falkner, and F.~Hutter.
\newblock Bayesian optimization with robust {B}ayesian neural networks.
\newblock \emph{NeurIPS}, 2016.

\bibitem[Srinivas et~al.(2010)Srinivas, Krause, Kakade, and Seeger]{srinivas2010gaussian}
N.~Srinivas, A.~Krause, S.~Kakade, and M.~Seeger.
\newblock {G}aussian process optimization in the bandit setting: No regret and experimental design.
\newblock In \emph{ICML}, 2010.

\bibitem[Thebelt et~al.(2021)Thebelt, Kronqvist, Mistry, Lee, Sudermann-Merx, and Misener]{thebelt2021entmoot}
A.~Thebelt, J.~Kronqvist, M.~Mistry, R.~M. Lee, N.~Sudermann-Merx, and R.~Misener.
\newblock {ENTMOOT}: A framework for optimization over ensemble tree models.
\newblock \emph{Computers \& Chemical Engineering}, 151:\penalty0 107343, 2021.

\bibitem[Thebelt et~al.(2022)Thebelt, Wiebe, Kronqvist, Tsay, and Misener]{thebelt2022maximizing}
A.~Thebelt, J.~Wiebe, J.~Kronqvist, C.~Tsay, and R.~Misener.
\newblock Maximizing information from chemical engineering data sets: Applications to machine learning.
\newblock \emph{Chemical Engineering Science}, 252:\penalty0 117469, 2022.

\bibitem[Tsay et~al.(2021)Tsay, Kronqvist, Thebelt, and Misener]{tsay2021partition}
C.~Tsay, J.~Kronqvist, A.~Thebelt, and R.~Misener.
\newblock Partition-based formulations for mixed-integer optimization of trained {ReLU} neural networks.
\newblock In \emph{NeurIPS}, 2021.

\bibitem[Vishwanathan et~al.(2010)Vishwanathan, Schraudolph, Kondor, and Borgwardt]{vishwanathan2010graph}
S.~V.~N. Vishwanathan, N.~N. Schraudolph, R.~Kondor, and K.~M. Borgwardt.
\newblock Graph kernels.
\newblock \emph{The Journal of Machine Learning Research}, 11:\penalty0 1201--1242, 2010.

\bibitem[Wan et~al.(2021)Wan, Kenlay, Ru, Blaas, Osborne, and Dong]{wan2021adversarial}
X.~Wan, H.~Kenlay, B.~Ru, A.~Blaas, M.~Osborne, and X.~Dong.
\newblock Adversarial attacks on graph classifiers via {B}ayesian optimisation.
\newblock In \emph{NeurIPS}, 2021.

\bibitem[Wan et~al.(2023)Wan, Osselin, Kenlay, Ru, Osborne, and Dong]{wan2023bayesian}
X.~Wan, P.~Osselin, H.~Kenlay, B.~Ru, M.~A. Osborne, and X.~Dong.
\newblock Bayesian optimisation of functions on graphs.
\newblock \emph{NeurIPS}, 2023.

\bibitem[Wang et~al.(2023)Wang, Lozano, Cardonha, and Bergman]{wang2023optimizing}
K.~Wang, L.~Lozano, C.~Cardonha, and D.~Bergman.
\newblock Optimizing over an ensemble of trained neural networks.
\newblock \emph{INFORMS Journal on Computing}, 2023.

\bibitem[Wen et~al.(2020)Wen, Liu, Chen, Li, Bender, and Kindermans]{wen2020gcn}
W.~Wen, H.~Liu, Y.~Chen, H.~Li, G.~Bender, and P.-J. Kindermans.
\newblock Neural predictor for neural architecture search.
\newblock In \emph{ECCV}, 2020.

\bibitem[White et~al.(2020)White, Neiswanger, Nolen, and Savani]{white2020study}
C.~White, W.~Neiswanger, S.~Nolen, and Y.~Savani.
\newblock A study on encodings for neural architecture search.
\newblock In \emph{NeurIPS}, 2020.

\bibitem[White et~al.(2021{\natexlab{a}})White, Neiswanger, and Savani]{white2021bananas}
C.~White, W.~Neiswanger, and Y.~Savani.
\newblock {BANANAS}: Bayesian optimization with neural architectures for neural architecture search.
\newblock In \emph{AAAI}, 2021{\natexlab{a}}.

\bibitem[White et~al.(2021{\natexlab{b}})White, Nolen, and Savani]{white2021localsearch}
C.~White, S.~Nolen, and Y.~Savani.
\newblock Exploring the loss landscape in neural architecture search.
\newblock In \emph{UAI}, 2021{\natexlab{b}}.

\bibitem[White et~al.(2023)White, Safari, Sukthanker, Ru, Elsken, Zela, Dey, and Hutter]{white2023NASthousand}
C.~White, M.~Safari, R.~Sukthanker, B.~Ru, T.~Elsken, A.~Zela, D.~Dey, and F.~Hutter.
\newblock Neural architecture search: insights from 1000 papers.
\newblock \emph{arXiv preprint arXiv:2301.08727}, 2023.

\bibitem[Wu et~al.(2019)Wu, Dai, Zhang, Wang, Sun, Wu, Tian, Vajda, Jia, and Keutzer]{wu2019fbnet}
B.~Wu, X.~Dai, P.~Zhang, Y.~Wang, F.~Sun, Y.~Wu, Y.~Tian, P.~Vajda, Y.~Jia, and K.~Keutzer.
\newblock {FBNet}: Hardware-aware efficient convnet design via differentiable neural architecture search.
\newblock In \emph{CVPR}, 2019.

\bibitem[Xie et~al.(2024)Xie, Zhang, Paulson, and Tsay]{xie2024global}
Y.~Xie, S.~Zhang, J.~Paulson, and C.~Tsay.
\newblock Global optimization of {G}aussian process acquisition functions using a piecewise-linear kernel approximation.
\newblock \emph{arXiv preprint arXiv:2410.16893}, 2024.

\bibitem[Xie et~al.(2025)Xie, Zhang, Qing, Misener, and Tsay]{xie2025bogrape}
Y.~Xie, S.~Zhang, J.~Qing, R.~Misener, and C.~Tsay.
\newblock {BoGrape}: Bayesian optimization over graphs with shortest-path encoded.
\newblock \emph{arXiv preprint arXiv:2503.05642}, 2025.

\bibitem[Ying et~al.(2019)Ying, Klein, Christiansen, Real, Murphy, and Hutter]{ying2019bench}
C.~Ying, A.~Klein, E.~Christiansen, E.~Real, K.~Murphy, and F.~Hutter.
\newblock {NAS-Bench-101}: Towards reproducible neural architecture search.
\newblock In \emph{ICML}, 2019.

\bibitem[Zhang et~al.(2023)Zhang, Campos, Feldmann, Walz, Sandfort, Mathea, Tsay, and Misener]{zhang2023optimizing}
S.~Zhang, J.~S. Campos, C.~Feldmann, D.~Walz, F.~Sandfort, M.~Mathea, C.~Tsay, and R.~Misener.
\newblock Optimizing over trained {GNNs} via symmetry breaking.
\newblock In \emph{NeurIPS}, 2023.

\bibitem[Zhang et~al.(2024)Zhang, Campos, Feldmann, Sandfort, Mathea, and Misener]{zhang2024augmenting}
S.~Zhang, J.~S. Campos, C.~Feldmann, F.~Sandfort, M.~Mathea, and R.~Misener.
\newblock Augmenting optimization-based molecular design with graph neural networks.
\newblock \emph{Computers \& Chemical Engineering}, 186:\penalty0 108684, 2024.

\bibitem[Zoph and Le(2017)]{zoph2017neural}
B.~Zoph and Q.~Le.
\newblock Neural architecture search with reinforcement learning.
\newblock In \emph{ICLR}, 2017.

\end{thebibliography}

\newpage
\appendix
\onecolumn
\section{Shortest path encoding}\label{app:shortest_path_encoding}

\subsection{Encoding}\label{app:encoding}
For each variable $\mathit{Var}$ in Table \ref{tab:notations} , we use $\mathit{Var}(G)$ to denote its value on a given graph $G$. For example, $d_{u,v}(G)$ is the shortest distance from node $u$ to node $v$ in graph $G$. If graph $G$ is given, all variable values can be easily obtained using classic shortest-path algorithms like the Floyd–Warshall algorithm \citep{floyd1962algorithm}. However, for graph optimization, those variables need to be constrained properly so that they have correct values for any given graph. In this section, we first provide a list of necessary conditions that these variables should satisfy based on their definitions. Then we will prove that these conditions are sufficient in next section.

\textbf{Condition ($\mathcal C1$):} At least $n_0$ nodes exist. W.l.o.g., assume that nodes with smaller indexes exist:
\begin{equation*}
    \left\{
    \begin{aligned}
    \sum\limits_{v\in [n]}A_{v,v}&\ge n_0&&\\
    A_{v,v}&\ge A_{v+1,v+1},&&\forall v\in [n-1]
    \end{aligned}
    \right.
\end{equation*}

\textbf{Condition ($\mathcal C2$):} Initialization for nonexistent nodes, i.e., if either node $u$ or node $v$ does not exist, edge $u\to v$ cannot exists, node $u$ cannot reach node $v$, and the shortest distance from node $u$ to node $v$ is infinity, i.e., $n$:
\begin{equation*}
    \begin{aligned}
        \min(A_{u,u},A_{v,v})=0\Rightarrow A_{u,v}=0,~r_{u,v}=0,~d_{u,v}=n,~\forall u,v\in [n],~u\neq v
    \end{aligned}
\end{equation*}
which could be rewritten as the following linear constraints:
\begin{equation*}
    \left\{
    \begin{aligned}
        2\cdot A_{u,v}&\le A_{u,u}+A_{v,v},&&\forall u,v\in [n],~u\neq v\\
        2\cdot r_{u,v}&\le A_{u,u}+A_{v,v},&&\forall u,v\in [n],~u\neq v\\
        d_{u,v}&\ge n\cdot (1-A_{u,u}),&&\forall u,v\in [n],~u\neq v\\
        d_{u,v}&\ge n\cdot (1-A_{v,v}),&&\forall u,v\in [n],~u\neq v\\
    \end{aligned}
    \right.
\end{equation*}

\textbf{Condition ($\mathcal C3$):} Initialization for single node, i.e., node $v$ can reach itself with shortest distance as $0$, and node $v$ is obviously the only node that appears in the shortest path from node $v$ to itself:
\begin{equation*}
    \left\{
    \begin{aligned}
        r_{v,v}&=1,&&\forall v\in [n]\\
        d_{v,v}&=0,&&\forall v\in [n]\\
        \delta_{v,v}^v&=1,&&\forall v\in [n]\\
        \delta_{v,v}^w&=0,&&\forall v,w\in [n],~v\neq w\\
    \end{aligned}
    \right.
\end{equation*}

\textbf{Condition ($\mathcal C4$):} Initialization for each edge, i.e., if edge $u\to v$ exists, node $u$ can reach node $v$ with shortest distance as $1$. Otherwise, the shortest distance from node $u$ to node $v$ is larger than $1$:
\begin{equation*}
    \begin{aligned}
        A_{u,v}=1&\Rightarrow r_{u,v}=1,~d_{u,v}=1,&&\forall u,v\in [n],~u\neq v\\
        A_{u,v}=0&\Rightarrow d_{u,v}>1,&&\forall u,v\in [n],~u\neq v\\
    \end{aligned}
\end{equation*}
which could be rewritten as the following linear constraints:
\begin{equation*}
    \left\{
    \begin{aligned}
        r_{u,v}&\ge A_{u,v},&&\forall u,v\in [n],~u\neq v\\
        d_{u,v}&\ge 2-A_{u,v},&&\forall u,v\in [n],~u\neq v\\
        d_{u,v}&\le 1+(n-1)\cdot (1-A_{u,v}),&&\forall u,v\in [n],~u\neq v\\
    \end{aligned}
    \right.
\end{equation*}    

\textbf{Condition ($\mathcal C5$):} Compatibility between distance and reachability, i.e,, node $u$ can reach node $v$ if and only the shortest distance from node $u$ to node $v$ is finite:
\begin{equation*}
    \begin{aligned}
        d_{u,v}<n\Leftrightarrow r_{u,v}=1,~\forall u,v\in[n],~u\neq v
    \end{aligned}
\end{equation*}
which could be rewritten as the following linear constraints:
\begin{equation*}
    \left\{
    \begin{aligned}
        d_{u,v}&\le n-r_{u,v},&&\forall u,v\in [n],~u\neq v\\
        d_{u,v}&\ge n-(n-1)\cdot r_{u,v},&&\forall u,v\in [n],~u\neq v\\
    \end{aligned}
    \right.
\end{equation*}  

\textbf{Condition ($\mathcal C6$):} Compatibility between path and reachability, i.e., (i) if node $w$ appears in the shortest path from node $u$ to node $v$, then node $u$ can reach node $w$, and node $w$ can reach node $v$ (the opposite is not always true), which means that node $u$ can reach node $v$ via node $w$:
\begin{equation*}
    \begin{aligned}
        \delta_{u,v}^w=1\Rightarrow r_{u,w}=r_{w,v}=1\Rightarrow r_{u,v}=1,\forall u,v,w\in [n],~u\neq v\neq w
    \end{aligned}
\end{equation*}
which could be rewritten as the following linear constraints:
\begin{equation*}
    \left\{
    \begin{aligned}
        r_{u,w}+r_{w,v}&\ge 2\cdot \delta_{u,v}^w,&&\forall u,v,w\in [n],~u\neq v\neq w\\
        r_{u,v}&\ge r_{u,w}+r_{w,v}-1,&&\forall u,v,w\in [n],~u\neq v\neq w
    \end{aligned}
    \right.
\end{equation*} 

\textbf{Condition ($\mathcal C7$):} Construction of shortest path, i.e., (i) always assume that both node $u$ and node $v$ appear in the shortest path from node $u$ to node $v$ for well-definedness, (ii) if edge $u\to v$ exists or node $u$ cannot reach node $v$, then no other nodes can appear in the shortest path from node $u$ to node $v$, (iii) if edge $u\to v$ does not exist but node $u$ can reach node $v$, then at least one node except for node $u$ and node $v$ will appear in the shortest path from node $u$ to node $v$:
\begin{equation*}
    \begin{aligned}
        \delta_{u,v}^u&=\delta_{u,v}^v=1,&&\forall u,v\in [n],~u\neq v\\
        A_{u,v}=1\lor r_{u,v}=0&\Rightarrow \sum\limits_{w\in [n]}\delta_{u,v}^w=2,&&\forall u,v\in [n],~u\neq v\\
        A_{u,v}=0\land r_{u,v}=1&\Rightarrow \sum\limits_{w\in [n]}\delta_{u,v}^w>2,&&\forall u,v\in [n],~u\neq v
    \end{aligned}
\end{equation*}
Observing that $A_{u,v}=1\lor r_{u,v}=0\Leftrightarrow r_{u,v}-A_{u,v}=0$ since $r_{u,v}\ge A_{u,v}$ always holds, we can rewrite these constraints as the following linear constraints:
\begin{equation*}
    \left\{
    \begin{aligned}
        \delta_{u,v}^u&=\delta_{u,v}^v=1,&&\forall u,v\in [n],~u\neq v\\
        \sum\limits_{w\in [n]}\delta_{u,v}^w&\ge 2+r_{u,v}-A_{u,v},&&\forall u,v\in [n],~u\neq v\\
        \sum\limits_{w\in [n]}\delta_{u,v}^w&\le 2+(n-2)\cdot(r_{u,v}-A_{u,v}),&&\forall u,v\in [n],~u\neq v
    \end{aligned}
    \right.
\end{equation*}

\textbf{Condition ($\mathcal C8$):} Triangle inequality of shortest distance, i.e., if node $u$ can reach node $w$ and node $w$ can reach node $v$, then the shortest distance from node $u$ to node $v$ is no larger than the shortest distance from node $u$ to node $w$ then to node $v$, and the equality holds when node $w$ appears in the shortest path from node $u$ to node $v$:
\begin{equation*}
    \begin{aligned}
        \delta_{u,v}^w=1&\Rightarrow d_{u,v}=d_{u,w}+d_{w,v},&&\forall u,v,w\in [n],~u\neq v\neq w\\
        r_{u,w}=r_{w,v}=1\land \delta_{u,v}^w=0&\Rightarrow d_{u,v}<d_{u,w}+d_{w,v},&&\forall u,v,w\in [n],~u\neq v\neq w
    \end{aligned}
\end{equation*}
where we omit $r_{u,w}=r_{w,v}=1$ in the first line since $\delta_{u,v}^w=1$ implies it. 

Similarly, we can rewrite these constraints as the following linear constraints:
\begin{equation*}
    \left\{
    \begin{aligned}
        d_{u,v}&\le d_{u,w}+d_{w,v}-(1-\delta_{u,v}^w)+(n+1)\cdot (2-r_{u,w}-r_{w,v}),&&\forall u,v,w\in [n],~u\neq v\neq w\\
        d_{u,v}&\ge d_{u,w}+d_{w,v}-2n\cdot (1-\delta_{u,v}^w),&&\forall u,v,w\in [n],~u\neq v\neq w
    \end{aligned}
    \right.
\end{equation*} 

Putting Conditions ($\mathcal C1$)--($\mathcal C8$) together presents the final formulation Eq.~\eqref{eq:final_MIP}.

\subsection{Theoretical guarantee}\label{app:theory}

All constraints in Eq.~\eqref{eq:final_MIP} are necessary conditions, i.e., as shown in Lemma \ref{lm:existence}.

\begin{lemma}\label{lm:existence}
    Given any labeled graph $G$ with $n$ nodes, $\{A_{u,v}(G),r_{u,v}(G),d_{u,v}(G),\delta_{u,v}^w(G)\}_{u,v,w\in [n]}$ is a feasible solution of Eq.~\eqref{eq:final_MIP} with $n_0=n$.
\end{lemma}
\begin{proof}
    By definition, it is easy to check that $\{A_{u,v}(G),r_{u,v}(G),d_{u,v}(G),\delta_{u,v}^w(G)\}_{u,v,w\in [n]}$ satisfies condition ($\mathcal C1$) -- ($\mathcal C8$). 
\end{proof}

The opposite is non-trivial to prove, that is, any feasible solution of Eq.~\eqref{eq:final_MIP} corresponds to an unique graph with $\sum_{v\in [n]}A_{v,v}$ nodes, which is guaranteed by Theorem \ref{thm:uniqueness}.

\begin{proof}[Proof of Theorem \ref{thm:uniqueness}]
    Denote $\mathcal F_{n_0,n}$ as the feasible  domain restricted by Eq.~\eqref{eq:final_MIP} with size $[n_0,n]$, and $\mathcal G_{n_0,n}$ as the whole graph space with node numbers in $[n_0,n]$. Define the following mapping:
    \begin{equation*}
        \begin{aligned}
            \mathcal M_{n_0,n}:\mathcal F_{n_0,n}&\to\mathcal G_{n_0,n}\\
            \{A_{u,v},r_{u,v},d_{u,v},\delta_{u,v}^w\}_{u,v,w\in [n]}&\mapsto \{A_{u,v}\}_{u,v\in [n]}
        \end{aligned}
    \end{equation*}
    For simplicity, we still use a $n\times n$ adjacency matrix to define a graph with node number less than $n$ and use $A_{v,v}(G)$ to represent the existence of node $v$. Also, the subscriptions, e.g., $\{\}_{u,v,w\in [n]}$, are omitted from now on. 
    
    If $n_1=\sum_{v\in [n]}A_{v,v}<n$, Condition $(\mathcal C1)$ forces that:
    \begin{equation*}
        \begin{aligned}
            A_{v,v}=
            \begin{cases}
            1,&v\in [n_1]\\
            0,&v\in [n]\backslash [n_0]
            \end{cases}
        \end{aligned}
    \end{equation*}
    For any pair of $(u,v)$ with $u\neq v$ and $\max(u,v)\ge n_1$, Conditions $(\mathcal C2)$ and $(\mathcal C7)$ uniquely define $\{r_{u,v},d_{u,v},\delta_{u,v}^w\}$ as:
    \begin{equation*}
        \begin{aligned}
            r_{u,v}=0,~d_{u,v}=n,~\delta_{u,v}^w=\begin{cases}
                1,&w\in  \{u,v\}\\
                0,&w\not\in \{u,v\}
            \end{cases}
        \end{aligned}
    \end{equation*}
    Therefore, it is equivalent to show that $\mathcal M_{n,n}$ is a bijection. Since Lemma \ref{lm:existence} already shows that $\mathcal M_{n,n}$ is a surjection, it suffices to prove that $\mathcal M_{n,n}$ is an injection. Precisely, for any feasible solution $\{A_{u,v},r_{u,v},d_{u,v},\delta_{u,v}^w\}$, there exists a graph $G$ with adjacency matrix given by $\{A_{u,v}(G)\}=\{A_{u,v}\}$, such that:
    \begin{equation}\label{eq:bijection}\tag{$\star$}
        \begin{aligned}
            \{r_{u,v}(G),d_{u,v}(G),\delta_{u,v}^w(G)\}=\{r_{u,v},d_{u,v},\delta_{u,v}^w\}
        \end{aligned}
    \end{equation}
    Since $r_{v,v}(G), d_{v,v}(G),\delta_{v,v}^w(G),\delta_{u,v}^u(G),\delta_{u,v}^v(G)$ are defined for completeness of our definition, whose variable counterparts are properly and uniquely defined in Condition $(\mathcal C3)$ and the first part of Condition $(\mathcal C7)$, we only need to consider all triples $(u,v,w)$ with $u\neq v\neq w$, which will not be specified later for simplicity.

    Now we are going to prove Eq.~\eqref{eq:bijection} holds by induction on $\min(d_{u,v}(G),d_{u,v})<n$.

    When $\min(d_{u,v}(G),d_{u,v})=1$, for any pair of $(u,v)$, we have:
    \begin{equation*}
        \begin{aligned}
            d_{u,v}(G)=1 &\Rightarrow A_{u,v}(G)=1,~r_{u,v}(G)=1,~\delta_{u,v}^w(G)=0 &&\longleftarrow \text{definition}\\
            &\Rightarrow A_{u,v}=1 &&\longleftarrow \text{definition of $\mathcal M_{n,n}$}\\
            &\Rightarrow r_{u,v}=1,~d_{u,v}=1,~\delta_{u,v}^w=0 &&\longleftarrow \text{Conditions $(\mathcal C4) + (\mathcal C7)$}
        \end{aligned}
    \end{equation*}
    and:
    \begin{equation*}
        \begin{aligned}
            d_{u,v}=1 &\Rightarrow A_{u,v}=1,~r_{u,v}=1,~\delta_{u,v}^w=0 &&\longleftarrow \text{Conditions $(\mathcal C4) + (\mathcal C7)$}\\
            &\Rightarrow A_{u,v}(G)=1 &&\longleftarrow \text{definition of $\mathcal M_{n,n}$}\\
            &\Rightarrow r_{u,v}(G)=1,~d_{u,v}(G)=1,~\delta_{u,v}^w(G)=0 &&\longleftarrow \text{definition}
        \end{aligned}
    \end{equation*}
    For both cases, we have Eq.~\eqref{eq:bijection} holds.

    Assume that Eq.~\eqref{eq:bijection} holds for any pair of $(u,v)$ with $\min(d_{u,v}(G),d_{u,v})\le sd$ with $sd<n-1$. Consider the following two cases for $\min(d_{u,v}(G),d_{u,v})=sd+1<n$.

    \textbf{Case I:} If $d_{u,v}(G)=sd+1$, we know that $r_{u,v}(G)=1$ since the shortest distance from node $u$ to node $v$ is finite. For any $w\not\in\{u,v\}$ such that $\delta_{u,v}^w(G)=1$, we have:
    \begin{equation*}
        \begin{aligned}
            \delta_{u,v}^w(G)=1 &\Rightarrow d_{u,w}(G)+d_{w,v}(G)=d_{u,v}(G)&&\longleftarrow\text{definition of $\delta_{u,v}^w(G)$} \\
             &\Rightarrow \max(d_{u,w}(G),d_{w,v}(G))\le sd &&\longleftarrow \text{$d_{u,w}(G)>0, d_{w,v}(G)>0$}\\
             &\Rightarrow d_{u,w}=d_{u,w}(G),~d_{w,v}=d_{w,v}(G) &&\longleftarrow\text{assumption of induction}\\
             & \Rightarrow r_{u,w}=r_{w,v}=1&&\longleftarrow\text{Condition $(\mathcal C5)$} \\
             &\Rightarrow d_{u,v}\le d_{u,w}+d_{w,v}=sd+1&&\longleftarrow\text{Condition $(\mathcal C8)$}\\
             &\Rightarrow d_{u,v}=sd+1 &&\longleftarrow\text{$d_{u,v}\ge sd+1$}\\
             &\Rightarrow r_{u,v}=1,~\delta_{u,v}^w=1&&\longleftarrow\text{Conditions $(\mathcal C5)+(\mathcal C8)$}
        \end{aligned}
    \end{equation*}
    which means that $r_{u,v}=r_{u,v}(G),~d_{u,v}=d_{u,v}(G),~\delta_{u,v}^w=\delta_{u,v}^w(G)$ with $\delta_{u,v}^w(G)=1$.
    
    For any $w\not\in\{u,v\}$ such that $\delta_{u,v}^w(G)=0$. If $\delta_{u,v}^w=1$, then we have:
    \begin{equation*}
        \begin{aligned}
            \delta_{u,v}^w=1&\Rightarrow r_{u,w}=r_{w,v}=1,~d_{u,v}=d_{u,w}+d_{w,v}&&\longleftarrow\text{Conditions $(\mathcal C6)+(\mathcal C8)$}\\
            &\Rightarrow \max(d_{u,w},d_{w,v})\le sd &&\longleftarrow\text{$d_{u,w}>0,~d_{w,v}>0$}\\
            &\Rightarrow d_{u,w}(G)=d_{u,w},~d_{w,v}(G)=d_{w,v}&&\longleftarrow\text{assumption of induction}\\
            &\Rightarrow d_{u,w}(G)+d_{w,v}(G)=sd+1=d_{u,v}(G)&&\longleftarrow\text{$d_{u,v}(G)=d_{u,v}=sd+1$}\\
            &\Rightarrow\delta_{u,v}^w(G)=1&&\longleftarrow\text{definition of $\delta_{u,v}^w(G)$}
        \end{aligned}
    \end{equation*}
    which contradicts to $\delta_{u,v}^w(G)=0$. Thus $\delta_{u,v}^w=0=\delta_{u,v}^w(G)$ with $\delta_{u,v}^w(G)=0$.

    \textbf{Case II:} If $d_{u,v}=sd+1$, from $sd+1>1$ and Condition $(\mathcal C4)$ we know that $A_{u,v}=0$, from Condition $(\mathcal C5)$ we have $r_{u,v}=1$, and then from Condition $(\mathcal C7)$ we obtain that $\sum_{w\in [n]}\delta_{u,v}^w>2$. 

    For any $w\not\in\{u,v\}$ such that $\delta_{u,v}^w=1$, we have:
    \begin{equation*}
        \begin{aligned}
            \delta_{u,v}^w=1 &\Rightarrow d_{u,w}+d_{w,v}=d_{u,v}=sd+1&&\longleftarrow\text{Condition $(\mathcal C8)$} \\
             &\Rightarrow \max(d_{u,w},d_{w,v})\le sd &&\longleftarrow \text{$d_{u,w}>0, d_{w,v}>0$}\\
             &\Rightarrow d_{u,w}(G)=d_{u,w},~d_{w,v}(G)=d_{w,v} &&\longleftarrow\text{assumption of induction}\\
             &\Rightarrow d_{u,v}(G)\le d_{u,w}(G)+d_{w,v}(G)=sd+1&&\longleftarrow\text{definition of $d_{u,v}(G)$}\\
             &\Rightarrow d_{u,v}(G)=sd+1 &&\longleftarrow\text{$d_{u,v}(G)\ge sd+1$}\\
             &\Rightarrow r_{u,v}(G)=1,~\delta_{u,v}^w(G)=1&&\longleftarrow\text{definition}
        \end{aligned}
    \end{equation*}
    which means that $r_{u,v}(G)=r_{u,v},~d_{u,v}(G)=d_{u,v},~\delta_{u,v}^w(G)=\delta_{u,v}^w$ with $\delta_{u,v}^w=1$.

    For any $w\not\in\{u,v\}$ such that $\delta_{u,v}^w=0$. If $\delta_{u,v}^w(G)=1$, then we have:
    \begin{equation*}
        \begin{aligned}
            \delta_{u,v}^w(G)=1&\Rightarrow d_{u,v}(G)=d_{u,w}(G)+d_{w,v}(G)&&\longleftarrow\text{definition of $\delta_{u,v}^w(G)$}\\
            &\Rightarrow \max(d_{u,w}(G),d_{w,v}(G))\le sd &&\longleftarrow\text{$d_{u,w}(G)>0,~d_{w,v}(G)>0$}\\
            &\Rightarrow d_{u,w}=d_{u,w}(G),~d_{w,v}=d_{w,v}(G)&&\longleftarrow\text{assumption of induction}\\
            &\Rightarrow d_{u,w}+d_{w,v}=sd+1=d_{u,v}&&\longleftarrow\text{$d_{u,v}(G)=d_{u,v}=sd+1$}\\
            &\Rightarrow\delta_{u,v}^w=1&&\longleftarrow\text{Condition $(\mathcal C8)$}
        \end{aligned}
    \end{equation*}
    which contradicts to $\delta_{u,v}^w=0$. Thus $\delta_{u,v}^w(G)=0=\delta_{u,v}^w$ with $\delta_{u,v}^w=0$.

    The remaining case is $d_{u,v}(G)=d_{u,v}=n$, i.e., node $u$ cannot reach node $v$. It is straightforward to verify that:
    \begin{equation*}
        \begin{aligned}
            r_{u,v}&=0=r_{u,v}(G)&&\longleftarrow\text{Condition $(\mathcal C5)$, definition of $r_{u,v}(G)$}\\
            \delta_{u,v}^w&=0=\delta_{u,v}^w(G)&&\longleftarrow\text{Condition $(\mathcal C7)$, definition of $\delta_{u,v}^w(G)$}
        \end{aligned}
    \end{equation*}
    
    Therefore, Eq.~\eqref{eq:bijection} always holds, which completes the proof.
\end{proof}

\section{Kernel encoding}\label{app:kernel_encoding}
This section presents encoding for shortest-graph kernels and binary node labels proposed in \citep{xie2025bogrape}. Notations are slightly changed to keep consistency with this paper.

\textbf{Graph kernel encoding} Introduce indicator variables $p^{u,v}_{s,l_1,l_2}=\mathbf 1(F_{u,l_1}=1,~d_{u,v}=s,~F_{v,l_2}=1)$ and count the number of each type of paths as:
\begin{equation*}
    \begin{aligned}
        P_{s,l_1,l_2}(G^i)=\sum\limits_{u,v\in [N]}p^{u,v}_{s,l_1,l_2}.
    \end{aligned}
\end{equation*}
Then SP kernel \eqref{eq:SP_kernel} could be formulated as:
\begin{equation*}
    \begin{aligned}
         k_g(G,G^i)&=\frac{1}{n^2n_i^2}\sum\limits_{s\in [n],l_1,l_2\in[L_n]}P_{s,l_1,l_2}(G^i)\cdot P_{s,l_1,l_2},\\
          k_g(G,G)&=\frac{1}{n^4}\sum\limits_{s\in [n],l_1,l_2\in [L_n]}P_{s,l_1,l_2}^2.
    \end{aligned}
\end{equation*}
To handle the quadratic term $P_{s,l_1,l_2}^2$, we further introduce indicator variables $P_{s,l_1,l_2}^{c}=\mathbf 1(P_{s,l_1,l_2}=c)$, and rewrite $k_g(G,G)$ as the following linear form:
\begin{equation*}
    \begin{aligned}
        k_g(G,G)=\frac{1}{n^4}\sum\limits_{s\in [n],l_1,l_2\in [L_n],c\in [n^2+1]}c^2\cdot P_{s,l_1,l_2}^c.
    \end{aligned}
\end{equation*}
Before formulating indicators $p_{s,l_1,l_2}^{u,v}$, we need indicators $d_{u,v}^s=\mathbf 1(d_{u,v}=s)$ that satisfy:
\begin{equation*}
    \sum\limits_{s\in [n+1]}d_{u,v}^s=1,~
    \sum\limits_{s\in [n+1]}s\cdot d_{u,v}^s=d_{u,v},~\forall u,v\in[n],
\end{equation*}
using which we can formulate $p_{s,l_1,l_2}^{u,v},~\forall u,v,s\in [n],~l_1,l_2\in [L_n]$ as:
\begin{equation*}
    \begin{aligned}
        3\cdot p_{s,l_1,l_2}^{u,v}\le F_{u,l_1}+d_{u,v}^s+F_{v,l_2},~p_{s,l_1,l_2}^{u,v}\ge F_{u,l_1}+d_{u,v}^s+F_{v,l_2}-2.
    \end{aligned}
\end{equation*}
Similar to $d_{u,v}^s$, indicators $P_{s,l_1,l_2}^c$ can be expressed as:
\begin{equation*}
    \begin{aligned}
        \sum\limits_{c\in [n^2+1]}P_{s,l_1,l_2}^c=1,~\sum\limits_{c\in [n^2+1]}c\cdot P_{s,l_1,l_2}^c=P_{s,l_1,l_2},~
        \forall s\in [n],~l_1,l_2\in [L_n].
    \end{aligned}
\end{equation*}

\textbf{Node label encoding} $k_n$ could be defined in multiple ways, \citet{xie2025bogrape} propose the following permutational-invariant kernel measuring the pair-wise similarity among node features:
\begin{equation*}
    \begin{aligned}
        k_n(F_n^1,F_n^2):=\frac{1}{n_1n_2L_n}\sum\limits_{v_1\in [n_1],v_2\in [n_2]}F_{v_1}^1\cdot F_{v_2}^2=\frac{1}{n_1n_2L_n}\sum\limits_{l\in [L_n]}N_l(F_n^1)\cdot N_1(F_n^2),
    \end{aligned}
\end{equation*}
where $N_l=\sum\limits_{v\in [n]}F_{v,l},~\forall l\in [L_n]$, and $n_1n_2L_n$ is the normalized coefficient.

Similar to the graph kernel encoding, we have:
\begin{equation*}
    \begin{aligned}
        k_n(F_n,F_n^i)&=\frac{1}{nn_iL_n}\sum\limits_{l\in [L_n]}N_l(F_n^i)\cdot N_l,\\
        k_n(F_n,F_n)&=\frac{1}{n^2L_n}\sum\limits_{l\in [L_n]}N_l^2=\frac{1}{n^2L_n}\sum\limits_{l\in [L_n],c\in [n+1]}c^2\cdot N_l^c,
    \end{aligned}
\end{equation*}
where indicators $N_l^c=\mathbf 1(N_l=c)$ satisfy:
\begin{equation*}
    \begin{aligned}
        \sum\limits_{c\in [n+1]}N_l^c=1,~\sum\limits_{c\in [n+1]}c\cdot N_l^c=N_l,~\forall l\in [L_n].
    \end{aligned}
\end{equation*}

\section{Experimental details and full results}\label{app:full_experiments}

\subsection{Hyperparameter settings in GP and BO}\label{app:param_setting}
We implement our graph kernels defined in Eqs. \eqref{eq:linear_kernel} and \eqref{eq:exponential_kernel} as an inherited \texttt{Kernel} class in GPflow \citep{matthews2017GPflow}. The initial values of the trainable kernel parameters $\alpha, \beta, \gamma$ and $\sigma_k^2$ are set to $1$ with bounds $[0.01, 100]$. In BO, we apply a batch setting to return $5$ architectures with the lowest LCB values in each iteration by setting Gurobi parameter \texttt{PoolSearchMode}=2. The final MIP model Eq. \eqref{eq:final_MIP} is designed for fixed graph size, but NAS-Bench-101 dataset consists of graph sizes ranging from $2$ to $7$. Our graph encoding supports changeable sizes. The only issue is that the normalized coefficients in kernel encoding are no longer constant, which complicates our formulation. One can resolve this issue by replacing these coefficients by constants or ignoring them. In NAS, however, architectures with more nodes usually have better performance. For instance, most high-quality architectures in NAS-Bench-101 have either 6 or 7 nodes. Therefore, in our experiments for NAS-Bench-101, we simply solve two MIP models with graph size set to $N=6, 7$ sequentially. Each model returns $5$ architectures, we still select 5 of 10 with the lowest LCB values. To encourage exploration, we set $\beta_t^{1/2}=3$ in LCB. The \texttt{TimeLimit} parameter in Gurobi for solving each MIP is set as 1800s. 

\subsection{Details on baselines}\label{app:baseline_details}
We provide more details on algorithms used in Section \ref{subsec:baselines}. We adapt the implementation from \citep{white2020study} for all baselines except for NAS-BOWL, where we use the publicly available code from \citep{ru2021interpretable}.

\begin{itemize}
    \item \textbf{Random:} Randomly sample the required number of architectures and evaluate them.
    \item \textbf{DNGO:} Deep Network for Global Optimization (DNGO) uses neural networks to learn an adaptive set of basis functions for Bayesian linear regression instead of GP in BO. It is adapted for NAS by treating the adjacency matrix of graph as encoding vector inputs.
    \item \textbf{BOHAMIANN:}  Bayesian Optimization with Hamiltonian Monte Carlo Artificial Neural
    Networks (BOHAMIANN) uses Bayesian neural networks as the surrogate model in both single- and multi-task BO, and achieves scalability through stochastic gradient Hamiltonian Monte Carlo. It is not originally designed for NAS but could be adapted by encoding graph input by adjacency matrix.
    \item \textbf{NASBOT:} Neural Architecture Search with Bayesian Optimisation and Optimal Transport (NASBOT) is a GP-based BO framework for NAS. It defines a distance metric to reveal the similarity between graphs called Optimal Transport Metrics for Architectures of Neural Networks (OTMANN). NASBOT specifically provides a list of operations for the evolutionary algorithm used in the acquisition function optimization.
    \item \textbf{Evolution:} Regularized evolution consists of mutating the best architectures from the population until a given budget runs out. \citet{white2020study} set the population size to 30 and outdate the architecture with the worst validation accuracy instead of the oldest one because it results in better performance in NAS tasks following. 
    \item \textbf{GP-BAYESOPT:} Standard BO with GP surrogate and UCB acquisition, implemented using ProBO \citep{neiswanger2019probo}. Similarity (distance) metric between two architectures is defined as the sum of Hamming distances between the adjacency matrices and the associated operations.
    \item \textbf{GCN:} Use Graph Convolutional Networks (GCN) as the neural predictor to predict the performance of random architectures and select the best $K$ samples for evaluation.
    \item \textbf{BONAS:} Bayesian Optimized Neural Architecture Search (BONAS) uses a GCN as surrogate model in BO to select multiple architectures in each iteration, and apply weight-sharing during the model training to accelerate traditional sampling methods.
    \item \textbf{Local search:} The simplest hill-climbing local search method evaluates all architectures in the neighborhood of a given sample. It is verified by \citet{white2021localsearch} that local search is a strong baseline in NAS when the noise in the benchmark datasets is reduced to a minimum.
    \item \textbf{BANANAS:} Bayesian
    optimization with neural architectures for NAS (BANANAS) uses a meta neural network over path encoding of individual architectures to predict the validation accuracies. The trained meta NN is used as the surrogate model in BO.
    \item \textbf{NAS-BOWL:} NAS-BOWL is a BO-based NAS algorithm which uses Weisfeiler Lehman (WL) graph kernel in GP surrogate model and adapts to both random sampling and mutation for optimizing the expected improvement (EI) acquisition function. Their experiment results show better performance of NAS-BOWL when using mutation as the acquisition function solver, hence we choose this setting to compare against. NAS-BOWL is considered as the state-of-the-art NAS algorithm.
\end{itemize}

\begin{figure}[t]
    \centering
    \includegraphics[width=\linewidth]{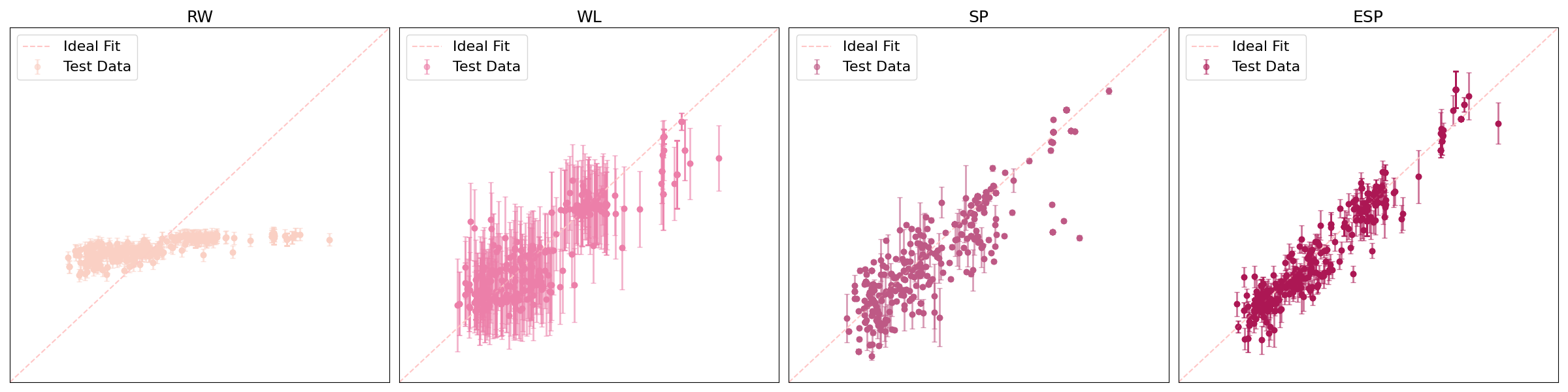}
    \caption{Predictive performance of graph GPs with different kernels. 50 and 400 architectures are randomly sampled from NAS-Bench-101 for training and testing resp. Predicted deterministic validation error are plotted against the true values, with one standard deviation as error bars.}
    \label{fig:N101_kernel_compare}
\end{figure}

\subsection{Additional graph BO for NAS results}\label{app:add_results}
We present additional experiment results on comparing NAS-GOAT with baselines when performing graph BO on NAS-Bench-101 and NAS-Bench-201 benchmarks. Figure \ref{fig:deterministic_appendix} shows the performance of the remaining baselines in deterministic setting, where NAS-GOAT still outperforms others in all cases. Figure \ref{fig:noisy_appendix} summarizes the comparisons between NAS-GOAT and all $11$ baselines in noisy setting. NAS-GOAT demonstrates comparable performance as state-of-the-art baselines, e.g. NAS-BOWL, NASBOT, BONAS. Note that the ``noisy'' setting is not another optimization task with noisy function evaluations, where we need to show NAS-GOAT still achieves the best objective function values. Instead, it is a NAS-specific setting where a good NAS algorithm is expected to find architectures with promising test accuracy even with unstable validation accuracy due to stochasticity in training process.

\begin{figure}
     \centering
     \includegraphics[width=\textwidth]{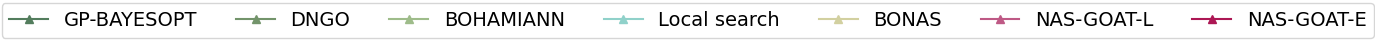}\\
     \vspace{.3mm}
     \begin{subfigure}[b]{0.245\textwidth}
         \centering
         \includegraphics[width=\textwidth]{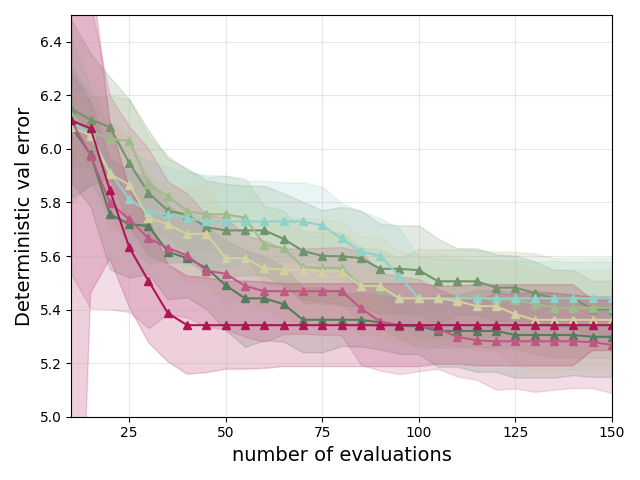}
         \caption{N101 (CIFAR10)}
     \end{subfigure}
     \hfill
    \begin{subfigure}[b]{0.245\textwidth}
         \centering
         \includegraphics[width=\textwidth]{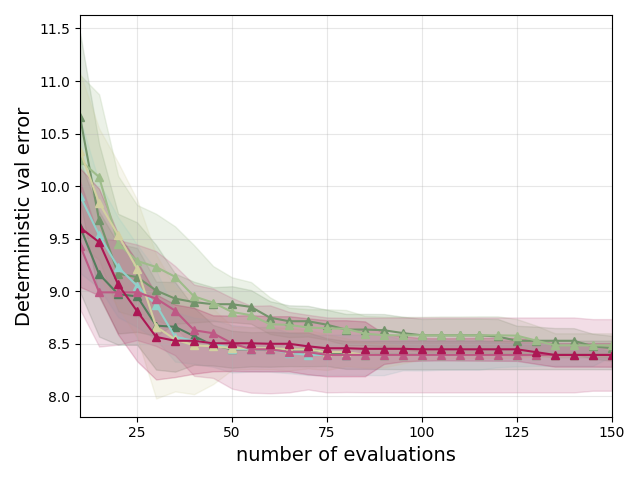}
         \caption{N201 (CIFAR10)}
     \end{subfigure}
     \hfill
     \begin{subfigure}[b]{0.245\textwidth}
         \centering
         \includegraphics[width=\textwidth]{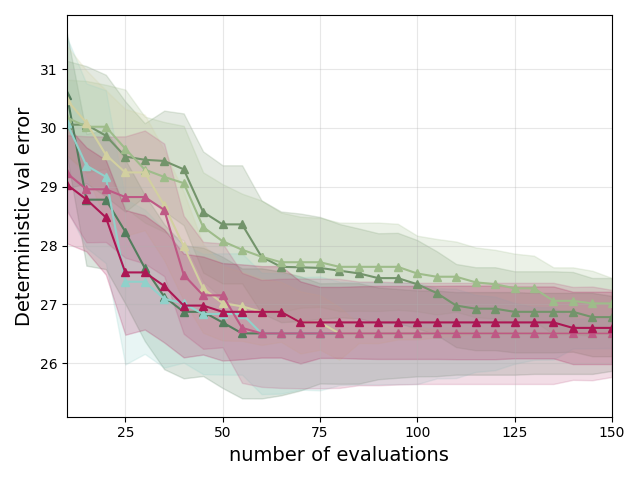}
         \caption{N201 (CIFAR100)}
     \end{subfigure}
    \hfill
     \begin{subfigure}[b]{0.245\textwidth}
         \centering
         \includegraphics[width=\textwidth]{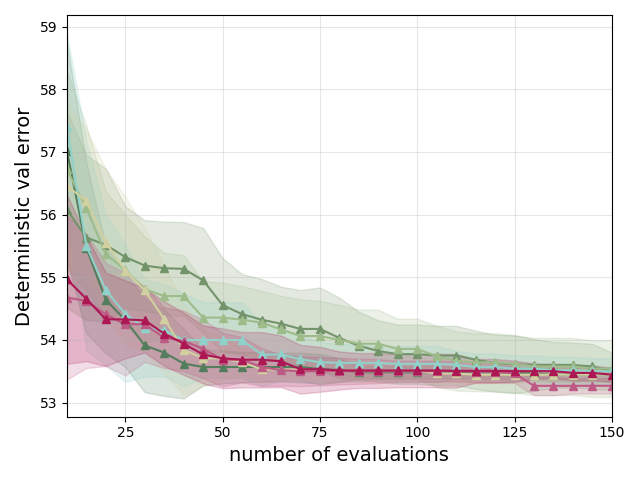}
         \caption{N201 (ImageNet)}
     \end{subfigure}
     \hfill
     \begin{subfigure}[b]{0.245\textwidth}
         \centering
         \includegraphics[width=\textwidth]{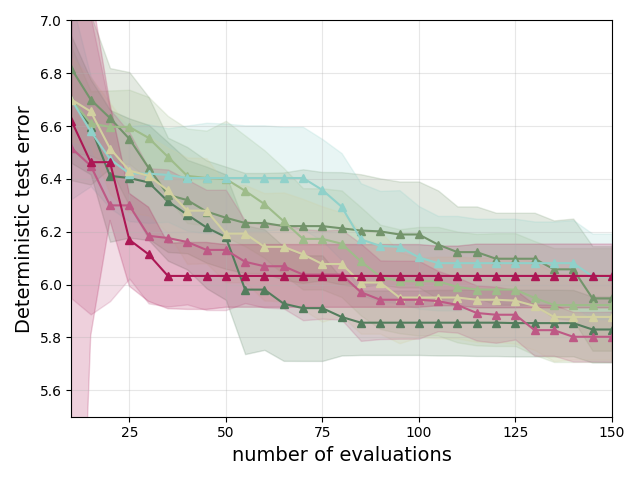}
         \caption{N101 (CIFAR10)}
     \end{subfigure}
     \hfill
     \begin{subfigure}[b]{0.245\textwidth}
         \centering
         \includegraphics[width=\textwidth]{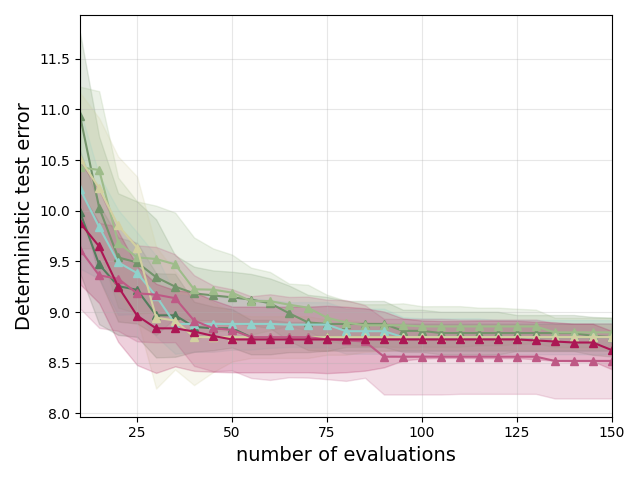}
         \caption{N201 (CIFAR10)}
     \end{subfigure}
     \hfill
     \begin{subfigure}[b]{0.245\textwidth}
         \centering
         \includegraphics[width=\textwidth]{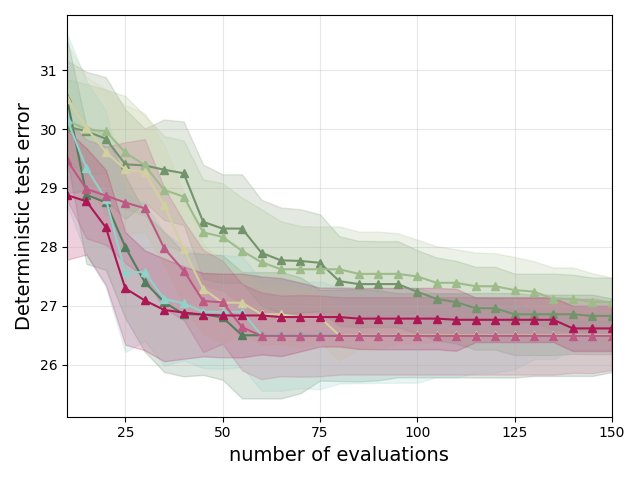}
         \caption{N201 (CIFAR100)}
     \end{subfigure}
     \hfill
     \begin{subfigure}[b]{0.245\textwidth}
         \centering
         \includegraphics[width=\textwidth]{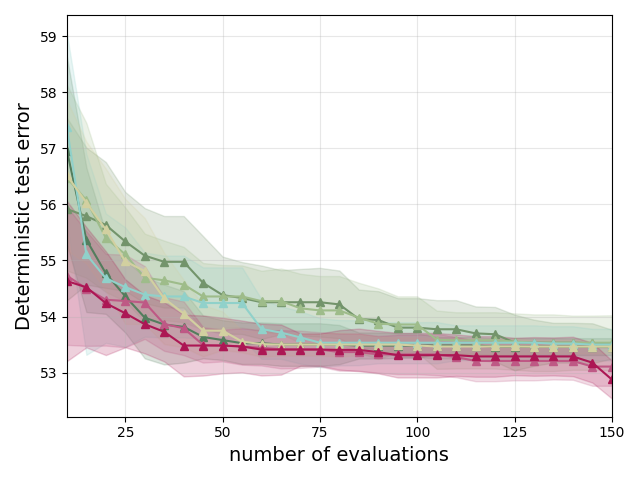}
         \caption{N201 (ImageNet)}
     \end{subfigure}
    \caption{Comparison NAS-GOAT with the remaining baselines. Numerical results of Graph BO on NAS-Bench-101 (N101) and NAS-Bench-201 (N201). (\textbf{Top}) Deterministic validation error. (\textbf{Bottom}) The corresponding test error. Median with one standard deviation over 20 replications is plotted.}
    \label{fig:deterministic_appendix}
\end{figure}

\begin{figure}
     \centering
     \includegraphics[width=\textwidth]{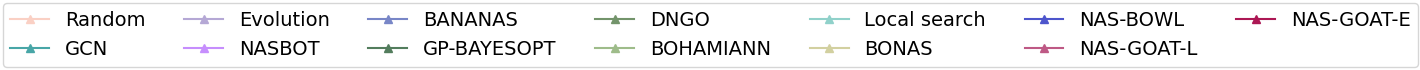}\\
     \vspace{.3mm}
     \begin{subfigure}[b]{0.245\textwidth}
         \centering
         \includegraphics[width=\textwidth]{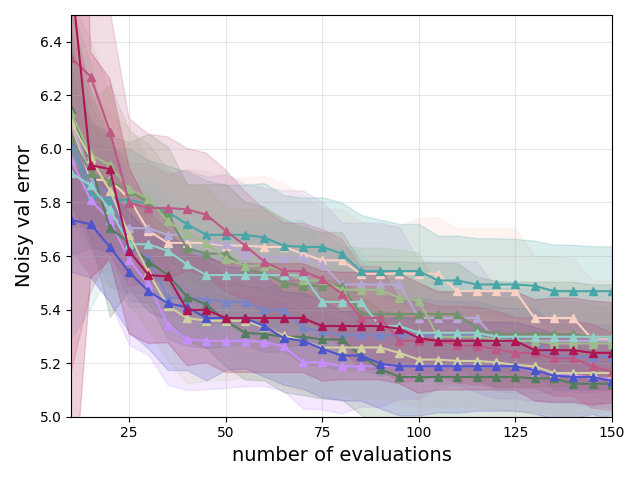}
         \caption{N101 (CIFAR10)}
     \end{subfigure}
     \hfill
    \begin{subfigure}[b]{0.245\textwidth}
         \centering
         \includegraphics[width=\textwidth]{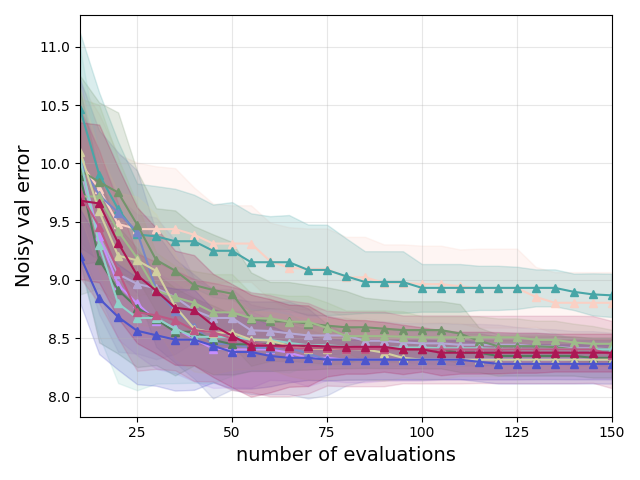}
         \caption{N201 (CIFAR10)}
     \end{subfigure}
     \hfill
     \begin{subfigure}[b]{0.245\textwidth}
         \centering
         \includegraphics[width=\textwidth]{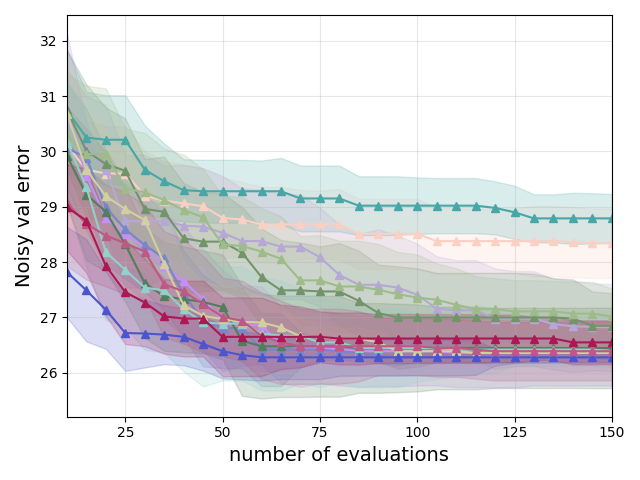}
         \caption{N201 (CIFAR100)}
     \end{subfigure}
    \hfill
     \begin{subfigure}[b]{0.245\textwidth}
         \centering
         \includegraphics[width=\textwidth]{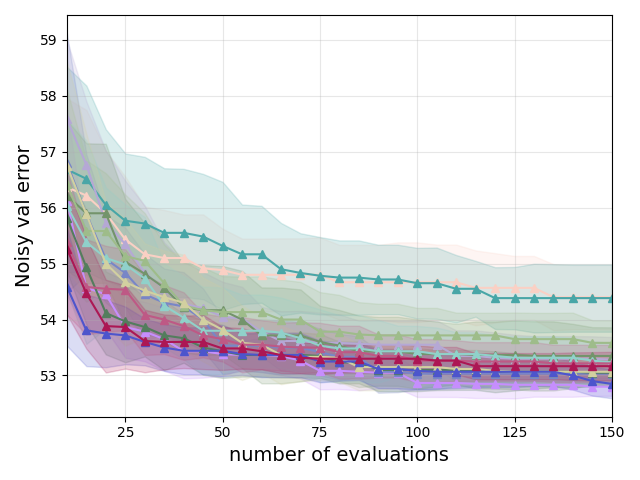}
         \caption{N201 (ImageNet)}
     \end{subfigure}
     \hfill
     \begin{subfigure}[b]{0.245\textwidth}
         \centering
         \includegraphics[width=\textwidth]{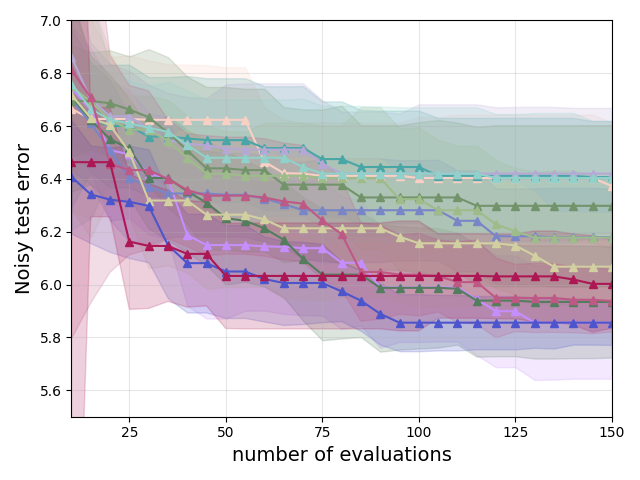}
         \caption{N101 (CIFAR10)}
     \end{subfigure}
     \hfill
     \begin{subfigure}[b]{0.245\textwidth}
         \centering
         \includegraphics[width=\textwidth]{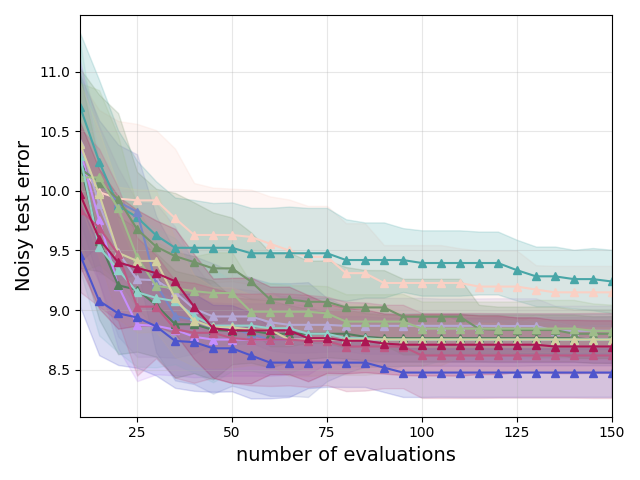}
         \caption{N201 (CIFAR10)}
     \end{subfigure}
     \hfill
     \begin{subfigure}[b]{0.245\textwidth}
         \centering
         \includegraphics[width=\textwidth]{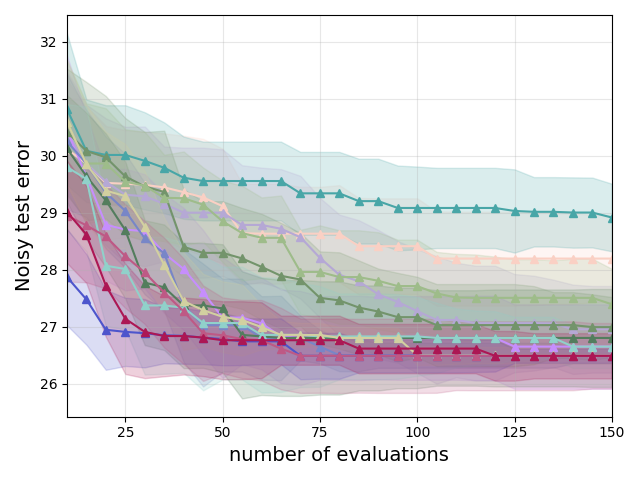}
         \caption{N201 (CIFAR100)}
     \end{subfigure}
     \hfill
     \begin{subfigure}[b]{0.245\textwidth}
         \centering
         \includegraphics[width=\textwidth]{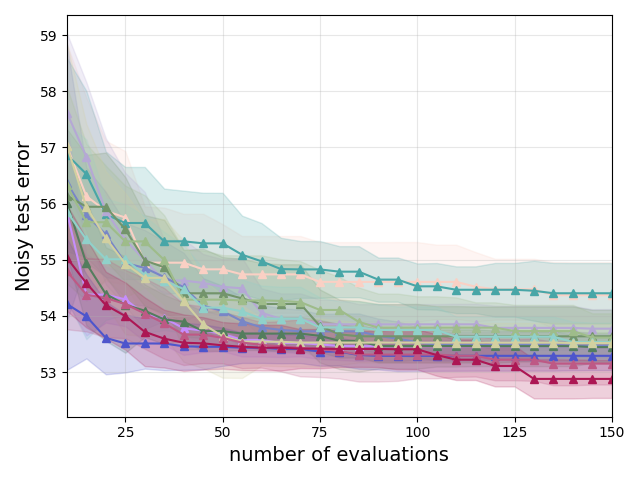}
         \caption{N201 (ImageNet)}
     \end{subfigure}
    \caption{Numerical results of Graph BO on NAS-Bench-101 (N101) and NAS-Bench-201 (N201). (\textbf{Top}) Noisy validation error. (\textbf{Bottom}) The corresponding test error. Median with one standard deviation over 20 replications is plotted.}
    \label{fig:noisy_appendix}
\end{figure}

\end{document}